\documentclass[11pt]{article} 

\usepackage[letterpaper,margin=1in]{geometry}

\usepackage[T1]{fontenc}
\usepackage{microtype}
\usepackage{wrapfig}
\usepackage{times}
\usepackage{algorithm}
\usepackage{algorithmic}
\usepackage{multirow}
\usepackage{diagbox,xcolor}
\usepackage{titlesec}
\titleformat*{\paragraph}{\bfseries}

\usepackage{amsthm,amsmath,amssymb,amsfonts,amssymb,mathtools}
\usepackage{amsmath,amssymb,amsfonts,amssymb,mathtools}
\usepackage{xcolor}
\usepackage{empheq}
\usepackage{dsfont}
\usepackage{mathrsfs}
\usepackage{graphicx}
\usepackage{enumitem}
\usepackage{verbatim}
\usepackage{aliascnt} 
\usepackage{xspace}
\usepackage{booktabs}
\usepackage{bbm}

\usepackage{caption}
\usepackage{tikz}
\usetikzlibrary{patterns}
\usetikzlibrary{arrows,shapes,automata,backgrounds,petri,positioning}
\usetikzlibrary{shadows}
\usetikzlibrary{calc}
\usetikzlibrary{spy}
\usetikzlibrary{angles, quotes}
\usetikzlibrary{matrix}
\usepackage{pgf,pgfplots}
\usepackage{pgfmath,pgffor}
\pgfplotsset{compat=1.17}

\usepackage{framed}

\definecolor[named]{ACMBlue}{cmyk}{1,0.1,0,0.1}
\definecolor[named]{ACMYellow}{cmyk}{0,0.16,1,0}
\definecolor[named]{ACMOrange}{cmyk}{0,0.42,1,0.01}
\definecolor[named]{ACMRed}{cmyk}{0,0.90,0.86,0}
\definecolor[named]{ACMLightBlue}{cmyk}{0.49,0.01,0,0}
\definecolor[named]{ACMGreen}{cmyk}{0.20,0,1,0.19}
\definecolor[named]{ACMPurple}{cmyk}{0.55,1,0,0.15}
\definecolor[named]{ACMDarkBlue}{cmyk}{1,0.58,0,0.21}

\usepackage[colorlinks,citecolor=blue,linkcolor=magenta,bookmarks=true]{hyperref}
\usepackage[nameinlink]{cleveref}
\crefname{ineq}{Inequality}{Inequality}
\creflabelformat{ineq}{#2{\upshape(#1)}#3}
\crefname{sub}{Subsection}{Subsection}
\creflabelformat{Subsection}{#2{\upshape(#1)}#3}
\crefname{sdp}{SDP}{SDP}
\creflabelformat{sdp}{#2{\upshape(#1)}#3}
\crefname{lp}{LP}{LP}
\creflabelformat{lp}{#2{\upshape(#1)}#3}
\usepackage{aliascnt}
\usepackage{cleveref}
\crefname{ineq}{Inequality}{Inequality}
\creflabelformat{ineq}{#2{\upshape(#1)}#3}
\crefname{sub}{Subsection}{Subsection}
\creflabelformat{Subsection}{#2{\upshape(#1)}#3}
\crefname{sdp}{SDP}{SDP}
\creflabelformat{sdp}{#2{\upshape(#1)}#3}
\crefname{lp}{LP}{LP}
\creflabelformat{lp}{#2{\upshape(#1)}#3}



\newtheorem{theorem}{Theorem}[section]

\newtheorem{lemma}[theorem]{Lemma}
\newtheorem{informal theorem}[theorem]{Theorem (informal statement)}

\newtheorem{corollary}[theorem]{Corollary}
\newtheorem{claim}[theorem]{Claim}
\newtheorem{fact}[theorem]{Fact}

\newtheorem{remark}[theorem]{Remark}

\newtheorem{definition}[theorem]{Definition}

\newcommand{\mgfi}{\exp\left(|A_i\cdot x+h_ix_i|^2/B\beta^2\right)}

\newcommand\norm[1]{\left\| #1 \right\|}

\renewcommand\vec[1]{\mathbf{#1}}

\DeclareMathOperator*{\E}{\mathbb{E}}

\def\multichoose#1#2{\ensuremath{\left(\kern-.3em\left(\genfrac{}{}{0pt}{}{#1}{#2}\right)\kern-.3em\right)}}

\newcommand{\di}{\partial_{i}}

\newcommand{\R}{\mathbb{R}}

\newcommand{\poly}{\mathrm{poly}}

\newcommand{\cube}[1]{\{\pm 1\}^{#1}}
\newcommand{\cevent}{\mathcal{E}}
\newcommand{\scevent}{\mathcal{E}_S}

\newcommand{\sign}{\mathrm{sign}}

\newcommand{\Ind}{\mathds{1}}
\newcommand{\1}{\Ind}
\newcommand{\boldone}{\mathbf{1}}

\newcommand{\x}{\vec x}

\newcommand{\citet}{\cite}
\newcommand{\citep}{\cite}

\def\tv{\textnormal{d}_{\mathsf{TV}}}
\def\kl{\textnormal{d}_{\mathsf{KL}}}
\newcommand{\Gmrf}{\dpsi_{G,\beta,t}}
\newcommand{\Rademacher}{\text{Rad}{(1/2)}}

\newcommand{\Gauss}{\mathcal{N}}

\newcommand{\ind}{\mathbbm{1}}
\newcommand{\dmrf}{D_{\psi}}

\newcommand{\dhamming}{d_{H}}
\newcommand{\csmooth}{$C$-\textit{smooth }}
\newcommand{\dpsi}{\mathcal{D}}
\newcommand{\parity}[1]{\chi_{#1}}
\newcommand{\sigmoid}{\sigma}
\newcommand{\diffsigmoid}{\left(\sigma(p(X))-\sigma(2\di \psi(X))\right)^2}

\newcommand{\diffpoly}{|p(X)-2\di \psi(X)|}
\title{Learning the Sherrington-Kirkpatrick Model \\ Even at Low Temperature}
\author{Gautam Chandrasekaran\thanks{\texttt{gautamc@cs.utexas.edu}. Supported by the NSF AI Institute for Foundations of Machine Learning (IFML).}\\
UT Austin
\and
Adam R. Klivans\thanks{\texttt{klivans@cs.utexas.edu}. Supported by the NSF AI Institute for Foundations of Machine Learning (IFML).}\\
UT Austin
}
\date{}
\begin{document}

\maketitle

\begin{abstract}

We consider the fundamental problem of learning the parameters of an undirected graphical model or Markov Random Field (MRF) in the setting where the edge weights are chosen at random.  For Ising models, we show that a multiplicative-weight update algorithm due to Klivans and Meka learns the parameters in polynomial time for any inverse temperature $\beta \leq \sqrt{\log n}$. 

This immediately yields an algorithm for learning the Sherrington-Kirkpatrick (SK) model beyond the high-temperature regime of $\beta < 1$.  Prior work breaks down at $\beta = 1$ and requires heavy machinery from statistical physics or functional inequalities.  In contrast, our analysis is relatively simple and uses only subgaussian concentration.

Our results extend to MRFs of higher order (such as pure $p$-spin models), where even results in the high-temperature regime were not known.

\end{abstract}
\newpage
\section{Introduction}
 In this work, we revisit the problem of structure learning and parameter recovery in undirected graphical models or Markov Random Fields (MRFs) over a binary alphabet. This class of distributions plays a significant role in various fields, including statistical physics, mathematics, and theoretical computer science.
Perhaps the simplest and most well-studied class of graphical models is the Ising model. An Ising model $D_{A,h}$ is a distribution over $\cube{n}$ given by the factorization
\[
 \Pr_{X\sim D_{A,h}}[X=x]\propto \exp\Biggr(\sum_{i<j}A_{ij}x_ix_j+\sum_{i=1}^{n}h_ix_i\Biggr).
 \]
 Here $A$ is called the interaction interaction matrix and $h$ is the external field. The dependency graph of $D_{A,h}$ is the set of edges $(i,j)$ for which $A_{ij}$ is non zero. 

The {\em structure learning} problem is to recover the underlying dependency graph $G$ given only iid samples from $D_{A,h}$.  We refer to the more difficult problem of recovering the edges of $G$ and their associated weights as {\em parameter recovery}.  The most commonly studied complexity measure for structure learning is the {\em width} of the Ising model denoted $\lambda(A,h) = \max_{i} (\sum_j |A_{ij}| + |h_i|$).  Klivans and Meka \cite{sparsitron} obtained the first algorithm with nearly optimal sample complexity and running time for learning Ising models (see also \cite{Bresler15,VuffrayMLC16,hamilton2017information}).  Their multiplicative-weight update algorithm {\em Sparsitron} uses $N = \exp(O(\lambda))\log n$ samples, runs in time $O(n^{2} N)$, and additionally recovers the parameters of the graph.

\subsection{Beyond Worst-Case Bounds and the SK Model}

Although the above bounds are best possible for all bounded-width Ising models \cite{santhanam2012information,sparsitron}, it is natural to ask whether the exponential dependence on $\lambda$ can be avoided for natural special cases of distributions.   A prominent example is the extensively studied \textit{Sherrington-Kirkpatrick} (SK) model.  The SK model is the Ising model with factorization 
$\frac{\beta}{\sqrt{n}}\sum_{i<j}A_{ij}x_ix_j+\sum_{i=1}^{n}h_ix_i$ where the entries of $A$ and $h$ are sampled iid from $\Gauss(0,1)$. 
We refer to this distribution as $\mathsf{SK}(\beta)$. 

When $\beta<1$, the SK model is said to be in the high-temperature regime, where various desirable properties such as replica symmetry hold \cite{talagrand2010mean}. The SK model experiences a phase transition at $\beta=1$, and $\beta\geq 1$ is known as the low temperature regime (see \cite{talagrand2003spin,panchenko2013sherrington} for a comprehensive discussion). Note that the underlying graph here is almost surely the complete graph, and so the structure learning problem is trivial. The parameter recovery problem, however, is still interesting. The problem had been first studied under the name \textit{spin glass inversion}, and heuristic methods were proposed in \cite{bachschmid2017statistical,mezard2009constraint}. Clearly, the width of this model is $O(\beta\sqrt{n})$ and thus the analysis from \cite{sparsitron} only guarantees a sample complexity and running time that is superpolynomial in $n$.

 For the easier problem (implied by parameter recovery) of recovering a distribution that is close in TV distance, \cite{anari2024universality} obtained a polynomial time algorithm for the limited portion of the high-temperature regime ($\beta<\frac{1}{4}$). 
 They rely on the \textit{approximate tensorization of entropy} of these models at high temperatures, which have many implications including the polynomial-time mixing of the Glauber dynamics of these distributions. Furthermore, their work relies on powerful results in functional inequalities \cite{bauerschmidt2019very,eldan2022spectral,adhikari2024spectral} that provably do not hold for $\beta>\frac{1}{4}$. 

To overcome the above barrier at $\beta=\frac{1}{4}$ and to attain parameter recovery, a very nice recent work of Gaitonde and Mossel \cite{gaitonde2024unified} takes a different approach. 
They first prove that bounding the operator norm of the covariance matrix of the SK model suffices for parameter recovery. 
They then apply a recent deep theorem on the boundedness of the operator norm of the covariance matrix of the SK model \cite{el2024bounds,brennecke2023operator,brennecke2024two} for the entire high-temperature regime $\beta<1$.  The boundedness of the operator norm provably does not hold for $\beta\geq 1$, and hence their techniques cannot work in the low-temperature regime. This suggests the existence of a barrier for efficient learnability beyond the high-temperature regime.

\subsection{Our Results}
We show that, surprisingly, there is no such barrier. We prove that Sparsitron run with $\exp(O(\beta^2+\beta\sqrt{\log n}))\cdot \poly(n)$ samples recovers the parameters of the SK model. In particular, for $\beta\leq \sqrt{\log n}$, the algorithm runs in polynomial time.  In doing so, we subsume and improve the corresponding result from \cite{gaitonde2024unified}.
\begin{theorem}[\Cref{thm:sk_param}, Informal]
    \label{thm:sk_model_theorem_informal}
    With probability $1-O(1/n)$ over $D_{A,h}\sim \mathsf{SK}(\beta)$, there exists an algorithm that takes $N=\exp(O(\beta^2+\beta\sqrt{\log n}))\cdot \poly(n,1/\epsilon)$ iid samples from $D_{A,h}$ and recovers a matrix $\widehat{A}$ and vector $\hat{h}$ such that $\norm{A-\widehat{A}}_{\infty}\leq \epsilon$ and $\norm{h-\widehat{h}}_{\infty}\leq \epsilon$. The running time is $O(n^2\cdot N)$.
\end{theorem}

As a corollary of the above theorem, we can output an Ising model that is close in TV distance to the ground truth.
\begin{corollary}[\Cref{clry:sk_tv}]
    With probability $1-O(1/n)$ over $D_{A,h}\sim \mathsf{SK}(\beta)$, there exists an algorithm that takes $\exp(O(\beta^2+\beta\sqrt{\log n}))\cdot \poly(n,1/\epsilon)$ iid samples from $D_{A,h}$ and outputs an Ising model $D_{\widehat{A},\widehat{h}}$ such that $\tv(D_{\widehat{A},\widehat{h}},D_{A,h})\leq \epsilon$. The running time is $O(n^2\cdot N)$.
\end{corollary}

We note that for $\beta<1$, the running time of \cite{gaitonde2024unified} is $\exp(O(1/(1-\beta)))\cdot \poly(n)$. For $\beta>1-o(1/\log n)$, this running time becomes super polynomial. Thus, our running time is better even in the high temperature regime for $\beta>1-o(1/\log n)$. A feature of our proof is that the Gaussian external field case does not require any additional work, whereas the corresponding proof of \cite{gaitonde2024unified} relies on a highly technical argument that does not apply to the full high-temperature regime. Finally, we note that the analysis of \cite{gaitonde2024unified} only implies learnability with probability $1-O(1/\log n)$ over the choice $D_{A,h}\sim \mathsf{SK}(\beta)$. On the other hand, our algorithm works with probability $1-O(1/n)$. 

In fact, our techniques solve a more general problem: given samples from an Ising model with arbitrary dependency graph $G$ and random weights, recover the graph and the parameters of the model. It is not clear if the techniques from prior work \cite{anari2024universality,gaitonde2024unified} solve this problem, even in the high-temperature regime. Formally, we work in the following model:

\begin{definition}[Random Ising Model]
    \label{defn:random_ising_model}
    Let $G$ be a graph on $n$ vertices with maximum degree $d$ and let $\beta>0$. We define  the distribution $\dpsi_{G,\beta}$ over Ising models as follows: a distribution $D\sim \dpsi_{G,\beta}$
    has weight matrix $A$ such that 
    \[
    A_{ij}=\begin{cases}
        \frac{\beta}{\sqrt{d+1}}\cdot \Gauss(0,1) \text{ (or $\frac{\beta}{\sqrt{d+1}}\cdot\Rademacher$)}& \text{ if $(i,j)$ is an edge in $G$}\\
        0 &\text{ otherwise}
    \end{cases}
    \] where each entry is sampled independently. Each entry of the external field $h$ is either $0$ or sampled from $\Gauss(0,1)$ (or $\Rademacher$). The parameter $\beta$ is called the inverse temperature. 
\end{definition}

Note that in the above definition, the graph $G$ is arbitrary and only the weights of the interaction are random. Some variants of the above definition have been studied under the names \textit{diluted spin glasses}~\cite{mezard1987spin,mezard2009information}, or the \textit{Edwards-Anderson} (EA) model \cite{edwards1975theory}.  When $G$ is instantiated to be the complete graph, the above definition corresponds to $\mathsf{SK}(\beta)$. We obtain the following theorem for parameter recovery. 

\begin{theorem}[\Cref{thm:random_ising_param}, Informal]
    \label{thm:random_ising_theorem_informal}
    With probability $1-O(1/n)$ over $D_{A,h}\sim \dpsi_{G,\beta}$ (with Gaussian weights), there exists an algorithm that takes $N=\exp(O(\beta^2+\beta\sqrt{\log n}))\cdot \poly(d,\log n,1/\epsilon)$ iid samples from $D_{A,h}$ and recovers a matrix $\widehat{A}$ and vector $\hat{h}$ such that $\norm{A-\widehat{A}}_{\infty}\leq \epsilon$ and $\norm{h-\widehat{h}}_{\infty}\leq n\epsilon$. The running time is $O(n^2\cdot N)$.
\end{theorem}
Note that if one is only interested in parameter recovery for the interaction matrix $A$, then our sample complexity is sub-polynomial in $n$ for bounded degree graphs.
In fact, when both the random interaction matrix and external field have rademacher entries, we can recover the parameters exactly (by rounding) with sample complexity sub-polynomial in $n$ and polynomial in $1/\beta$. 

\begin{theorem}[\Cref{thm:random_ising_param_rademacher}]
\label{thm:random_ising_rad_informal}
    With probability $1-O(1/n)$ over $D_{A,h}\sim \dpsi_{G,\beta}$ (with rademacher weights), there exists an algorithm that takes $N=\exp(O(\beta^2+\beta\sqrt{\log n}))\cdot O(d\log(n/\beta)/\beta^2)$ samples and recovers matrix $A$ and external field $h$ exactly. The running time is $O(n^2\cdot N)$. 
\end{theorem}
 The structure learning algorithm is immediate: as long as each non-zero edge has value at least $\eta$, run algorithm from \Cref{thm:random_ising_theorem_informal} with $\epsilon=\eta/2$ and add all edges $(i,j)$ with $\widehat{A}_{ij}\geq \eta/2$ to the graph.  In the Gaussian case, via standard Gaussian anticoncentration, with high probability, we have $\eta\geq \Omega(\beta/n^3)$ and thus we need polynomial sample complexity to identify these edges. In the rademacher case, since we exactly recover $A$, we just look at the non zero entries of $A$ to recover the graph.
\begin{corollary}
    With probability $1-O(1/n)$ over $D_{A,h}\sim \dpsi_{G,\beta}$ (with rademacher weights), there exists an algorithm that takes $N=\exp(O(\beta^2+\beta\sqrt{\log n}))\cdot O(d\log(n/\beta)/{\beta^2})$ iid samples from $D_{A,h}$ and recovers the graph $G$.
\end{corollary}
\begin{corollary}
    With probability $1-O(1/n)$ over $D_{A,h}\sim \dpsi_{G,\beta}$ (with Gaussian weights), there exists an algorithm that takes $N=\exp(O(\beta^2+\beta\sqrt{\log n}))\cdot \poly(n,\log n,1/\beta)$ iid samples from $D_{A,h}$ and recovers the graph $G$.
\end{corollary}
We also obtain closeness in TV distance as a direct corollary.
\begin{corollary}
     With probability $1-O(1/n)$ over $D_{A,h}\sim \dpsi_{G,\beta}$ (with Gaussian weights), there exists an algorithm that takes $\exp(O(\beta^2+\beta\sqrt{\log n}))\cdot \poly(n,1/\epsilon)$ iid samples from $D_{A,h}$ and outputs an Ising model $D_{\widehat{A},\widehat{h}}$ such that $\tv(D_{\widehat{A},\widehat{h}},D_{A,h})\leq \epsilon$. The running time is $O(n^2\cdot N)$.
\end{corollary}

\subsection{Our Results: Markov Random Fields} We now move on to the more general problem of parameter recovery and structure learning in higher order Markov Random Fields (MRFs) with random weights. We first give a formal definition of an MRF.
\begin{definition}[Markov Random Field]
    A distribution $\dmrf$ is a $t$-wise Markov Random Field (MRF) with dependency graph $G$ if 
    \begin{equation}
        \label{eqn:mrf_density}
        \Pr_{X\sim \dmrf}[X=x]\propto \exp\bigr(\psi(x)\bigr)=\exp\Biggr(\sum_{S\in C_{t}(G)}f_{S}(x)\Biggr)
    \end{equation}
    where $C_t(G)$ are the cliques of $G$ of size at most $t$ and $f_S$ are arbitrary functions on the variables of $S$. The function $\psi(\x)$ is a degree $t$ polynomial called  the factorization polynomial of $\dmrf$. 
\end{definition}
Note that Ising models are the special case when $t=2$. The dependency graph can also be defined as the following: the set of neighbors of vertex $v$ is the minimal set $S$ such that $X_v$ and $X_{[n]\setminus S\cup\{v\}}$ are independent conditioned on the variables in $S$. Furthermore, every distribution with such a dependency structure can be expressed as an MRF \cite{hammersley-clifford}. Now, the width of the model $\lambda(\psi)=\max_{i}\norm{\di\psi}_1$ where $\norm{p}_1$ for a polynomial $p$ is the sum of the absolute values of its coefficients. $\di\psi$ is the partial derivative of $\psi$ with respect to the index $i$. Similar to the case of the Ising model, Klivans and Meka \cite{sparsitron} give an algorithm that uses $N=\exp(O(\lambda))\log n$ samples, runs in time $O(n^t\cdot N)$ and recovers the graph $G$, see also (\cite{hamilton2017information}). 

The first MRF we study is a generalization of the SK model called a pure $t$-spin model that is also widely studied (see \cite{talagrand2010mean,panchenko2013sherrington} for a discussion). The problem of parameter recovery in these models for the full high-temperature region was stated as an open problem in \cite{gaitonde2024unified} (they state a more general problem for mixed $t$-spin models, but even the special case of pure $t$-spin models is open and suffers the same technical barriers). This is the MRF with factorization $\psi(x)=\frac{\beta}{n^{(t-1)/2}}\sum_{\alpha\in [n]^t}\widehat{\psi}_{\alpha}\prod_{j=1}^{t}x_{\alpha_j}$ where $\widehat{\psi}_{\alpha}$ are sample iid from $\Gauss(0,1)$ for every multi-index $\alpha\in [n]^t$. Note that $\lambda(\psi)=\Omega(n^{(t-1)/2})$ and hence the running time of Sparsitron is again exponentially large. 

For a limited range of the high-temperature regime, the results of \cite{anari2024universality} do solve the weaker problem of recovering a distribution that is close in total variation distance. However, their techniques will provably not work for the entire high-temperature range, as these models experience the property of \textit{shattering} in the high-temperature regime, and this provably rules out fast mixing and entropy factorization \cite{arous2024shattering,alaoui2023shattering,gamarnik2023shattering}. 

Observing this barrier, \cite{gaitonde2024unified} conjectured that boundedness of certain moment matrices are necessary for learning in the entire high-temperature regime.
We prove that this is not the case. 
We show that it is possible to recover parameters even in the low temperature regime where the aforementioned moment matrices are provably unbounded \cite{el2024bounds}. We now state our theorem on parameter recovery in pure $t$-spin models.

\begin{theorem}[\Cref{thm:pure_tspin_recovery}, Informal]
\label{thm:pure_tspin_recovery_informal}
    Let $\dpsi$ be a pure $t$-spin model with inverse temperature $\beta$. Then, we have that with probability at least $1-O(1/n^t)$ over $\dmrf\sim \dpsi$, there exists an algorithm that draws $N=\exp(O(\beta^2t^2+\beta t\sqrt{t\log n}))\cdot \poly(n^t,1/\epsilon)$ samples and runs in time $O(N\cdot n^t)$ that outputs a $t$-MRF $D_{\tilde{\psi}}$ such that (1) $\norm{\psi-\tilde{\psi}}_1\leq \epsilon$, (2) $\tv(D_{\psi},D_{\tilde{\psi}})\leq {\epsilon}$.
\end{theorem}

Again, our techniques solve the following more general problem: given iid samples from an MRF with dependency graph $G$ and random coefficients, find $G$ and recover the coefficients. Formally, we work in the following model:
\begin{definition}[Random MRF]
\label{defn:random_MRF}
    Let $G$ be a graph on $n$ vertices with maximum degree $d$. Let $\beta>0$ be the inverse temperature and $t>0$ be the degree of the MRF. The $t$-MRF $\psi$ sampled from the distribution over MRFs 
 $\dpsi_{G,\beta,t}$ is the MRF with factorization polynomial $\psi$ defined as 
 \[
 \psi(x)\coloneq \frac{\beta}{(d+1)^{(t-1)/2}}\sum_{S\in C_{t}(G)}\widehat{\psi}(S)\cdot\prod_{i\in S}x_i
 \] where $\{\widehat{\psi}(S)\}_{S\in C_{t}(G)}$ are sampled iid from $\Gauss(0,1)$ (or $\Rademacher$).
\end{definition}

We obtain the following theorems on parameter recovery and structure learning in these models.

\begin{theorem}[\Cref{thm:mrf_tv_recovery}, Informal]
    With probability $1-O(1/n^t)$ over $\dmrf\sim D_{G,\beta,t}$ (with Gaussian weights), there exists an algorithm that draws $N=\exp(O(\beta^2t+\beta t\sqrt{\log n}))\cdot \poly(n^t,1/\epsilon)$ iid samples from $\dmrf$ and runs in time $O(N\cdot n^t)$ that outputs a $t$-MRF $D_{\tilde{\psi}}$ such that (1) $\norm{\psi-\tilde{\psi}}_1\leq \epsilon$, (2) $\tv(D_{\psi},D_{\tilde{\psi}})\leq {\epsilon}$.
\end{theorem}
\begin{corollary}[\Cref{thm:structure_learning_random_mrf}, Informal]
    With probability $1-O(1/n^t)$ over $\dmrf\sim D_{G,\beta,t}$ (with Gaussian weights), there exists an algorithm that draws $N=\exp(O(\beta^2+\beta\sqrt{\log n}))\cdot \poly(n^t, 1/\beta))$ iid samples from $\dmrf$ and runs in time $O(N\cdot n^t)$ that recovers the graph $G$.
\end{corollary}

Similar to our results for the Ising model, we obtain improved bounds when the weights are rademacher. We give an algorithm that draws only a sub-polynomial number of samples and recovers $\psi$ exactly. 
\begin{theorem}[\Cref{thm:random_mrf_param_rademacher}, Informal]
    With probability $1-O(1/n^t)$ over $\dmrf\sim D_{G,\beta,t}$ (with rademacher weights), there exists an algorithm that draws $N=\exp(O(\beta^2+\beta\sqrt{\log n}))\cdot O(td^{t}\log (n/\beta)/\beta^2)$ iid samples from $\dmrf$ and runs in time $O(N\cdot n^t)$ that recovers the polynomial $\psi$ exactly.
\end{theorem}

\subsection{Related Work}
 The problem of designing efficient algorithms for learning undirected graphical models has a long history of research \cite{chow1968approximating,wainwright2006high,abbeel2006learning,bresler2008reconstruction,netrapalli2010greedy,tandon2014learning}.
Current work on structure learning and parameter recovery for MRFs primarily focuses on the setting where $\lambda$, the sum of the absolute values of each node's edge weights, is bounded.  In an important result, Bresler \cite{Bresler15} gave the first polynomial-time algorithm for constant $\lambda$ that did not take any further assumptions on the underlying graph (such as correlation decay).  His algorithm, however, has sample complexity that is doubly exponential in $\lambda$.  Bresler's result was extended to general MRFs by Hamilton et al.~\cite{hamilton2017information}.  Vuffray et al.~\cite{VuffrayMLC16} gave the first efficient algorithm for learning Ising models with nearly optimal sample complexity (singly exponential in $\lambda$) but with suboptimal running time.  

Using a different approach, Klivans and Meka \cite{sparsitron} obtained the first algorithm for learning Ising models with both near-optimal sample complexity and near-optimal running time. Their algorithm also extends to MRFs and non-binary alphabets.  
In a follow-up work, Wu et al.~\cite{wu2019sparse} showed how to replace the algorithm in \cite{sparsitron} with an algorithm for sparse logistic regression.  Their analysis follows the same framework as \cite{sparsitron} and results in a slightly improved sample-complexity bound for large alphabets.  The Wu et al. algorithm was extended to the case of MRFs in \cite{zhang2020privately} by closely following the analysis of the MRF case in \cite{sparsitron}. Vuffray et al.~\cite{vuffray2022efficient} extended these results to more general models. Other recent works that extend Ising and MRF learning to different settings include \cite{goel2019learning,prasad2020learning,moitra2021learning,diakonikolas2021outlier, dagan2021learning, bhattacharyya2021near, gaitonde2024efficiently}.
\section{Preliminaries}

Given a set $S\subseteq [n]$ and a vector $x\in \cube{n}$, we use $x_{S}$ to denote the vector obtained from $x$ by restricting to the indices in the set $S$. Similarly we use $x_{\overline{S}}$ to denote the $x$ restricted to coordinates outside $S$. We use $\parity{S}(x)$ to denote the monomial $\prod_{i\in S}x_{i}$. Given a polynomial $p$, $\widehat{p}(S)$ refers to the coefficient of $\parity{S}$ in $p$. $\Gauss(0,1)$ refers to the univariate standard normal distribution. $\Rademacher$ refers to the Rademacher random variable that takes $\pm 1$ with probability $1/2$. For $x,y\in \cube{n}$, the Hamming distance between $x$ and $y$ is defined as $d_{H}(x,y)\coloneq \sum_{i\in [n]}\ind\{x_i\neq y_i\}$. Given a polynomial $p=\sum_{S\subseteq [n]}\widehat{p}(S)\parity{S}(x)$, the partial derivative with respect to set $T$ is defined as $\partial_Tp(x)\coloneq \sum_{S\supseteq T}\widehat{p}(S)\parity{S}(x)$. When $S=\{i\}$, we use the notation $\di p$ for the polynomial $\partial_{\{i\}}p$. A maximal monomial of $p$ is any set $S\subseteq[n]$ such that $\widehat{p}(T)=0$ for all $T\supset S$. The sigmoid function $\sigma:\R\to \R$ is defined as $\sigmoid(x)\coloneq \frac{1}{1+e^{-x}}$.
\begin{fact}[Properties of subgaussian distributions,~\cite{hdp}]
    \label{fact:subgaussian}
For $\lambda>0$, a random variable $X$ is $\lambda$-subgaussian if there exists an absolute constant $C$ such that  $\Pr[|X|\geq t]\leq 2\cdot\exp(-(t/\lambda)^2);$ $\E[\exp((X/(\lambda C))^2)]\leq 2;$ and 
    $\E[\exp(rX)]\leq \exp(C^2r^2\lambda^2)$, for all $r\in \R$.
\end{fact}

\begin{fact}[Anti-Lipschitzness of the Sigmoid]
    \label{fact:antilipschitz_sigmoid}
    For $a,b\in \R$, we have that $|\sigmoid(a)-\sigmoid(b)|\geq \exp(-|a|-3)\cdot \min(1,|a-b|)$.

\end{fact}
\begin{fact}
    \label{fact:conditional_MRF}
    Let $D_{\psi}$ be a $t$-MRF with factorization polynomial $\psi$. Then, for $x\in \cube{n}$ and $i\in[n]$, we have that $\Pr_{X\sim D}\bigr[X_i=1\mid X_{[n]\setminus\{i\}}=x_{[n]\setminus\{i\}}\bigr]=\sigmoid(2\di \psi(x))$
\end{fact}

\section{Learning Random Ising Models}
\label{sec:ising}
In this section we present the argument for learning random Ising models with Gaussian external fields. We first review the Sparsitron algorithm and analysis from \cite{sparsitron}. We then overview the techniques used by Gaitonde and Mossel \cite{gaitonde2024unified} to solve the high-temperature case and show how we sidestep their technical barriers to achieve learnability in the low-temperature regime.

\subsection{The Analysis of Sparsitron}
Klivans and Meka \cite{sparsitron} proved that the problem of parameter recovery in Ising models can be reduced to learning sparse generalized linear models with noise. We now sketch their argument (we consider the zero external field case for simplicity). From the definition of the Ising measure, it holds that for any $i\in [n]$, $\Pr_{X\sim D_A}[X_i=1\mid X_{[n]\setminus\{i\}}=x]=\sigmoid(2A_i\cdot x)$. Recall $\sup_{i\in [n]}\norm{A_i}_1\leq \lambda$.  First, a multiplicative-weight update algorithm Sparsitron is run to obtain a vector $w$ to minimize $\E_{X\sim D_A}[(\sigmoid(w\cdot X)-\sigmoid(A_i\cdot X))^2]$.  
Formally, their multiplicative-weight update algorithm has the following guarantee:
\begin{theorem}[\cite{sparsitron}]
\label{thm:sparsitron}
    Let $\lambda,\epsilon,\delta>0$. Let $D$ be a distribution on $\cube{n}\times \cube{}$ where $\Pr[Y=+1|X]=\sigmoid(w\cdot X)$ for $(X,Y)\sim D$ and vector $w\in \R^{n}$ with $\norm{w}_1\leq \lambda$. There exists an algorithm that takes $N=O\bigr(\lambda^2(\ln(n/\delta\epsilon))/\epsilon^2\bigr)$ independent samples from $D$, runs in time $O(nN)$, and outputs a vector $\widehat{w}$ such that
    \[
    \E_{(X,Y)\sim D}\bigr[(\sigmoid(w\cdot X)-\sigmoid(\widehat{w}\cdot X))^2\bigr]\leq \epsilon
    \] with probability at least $1-\delta$.
\end{theorem}

To recover the parameters of $D_A$, they run the above algorithm for each $i\in[n]$ with $\epsilon$ set to $\exp({-\Omega(\lambda)})\epsilon^2$ and output the recovered vectors.  The final sample complexity is $N=\exp({O(\lambda)})\log (n/\epsilon)/\epsilon^4$ and the running time is $O(n^2\cdot N)$.  The source of their (unavoidable in the worst case) exponential dependence on $\lambda$ lies in their proof reducing parameter recovery to obtaining small squared loss.  The proof of this reduction goes through two steps.
\begin{itemize}
    \item ({\em Step 1}) Observe that the sigmoid function satisfies a weak anti-lipschitz property. Formally, we have for any $a,b\in \R$, $|\sigmoid(a)-\sigmoid(b)|\geq \exp({-|a|-3})\cdot \min(1,|a-b|)$. Now, since $|A_i\cdot X|\leq \lambda$ for any $x\in \cube{n}$, any $w$ that attains squared loss less than $\epsilon^2$ must satisfy the following property: 
\begin{equation}
\label{eqn:small_inner_product}
\E_{X\sim D_A}[|A_i\cdot X-w\cdot X|^2]\leq \exp({O(\lambda)})\epsilon^2.\end{equation}
\item ({\em Step 2}) Note that for any $j\in [n]$ and $y\in \cube{n-1}$, it holds that $\min_{b\in \cube{}}\Pr_{X\sim D_A}[X_j=b\mid X_{[n]\setminus \{j\}}=y]\geq \exp(-4|A_j\cdot y|)/2\geq \exp({-O(\lambda)})$. A distribution satisfying this property is said to be $\exp({-O(\lambda)})$-unbiased. 
Now, say a vector $w$ has $|w_{j}-A_{ij}|=\alpha$. Then, for any $y$ in $\cube{n}$, there exists a $b\in\cube{}$ such that $|b(w_{j}-A_{ij})+(A_i-w)_{[n]\setminus \{j\}}\cdot y)|\geq \alpha$. Using the fact that $D_A$ is $\exp({-O(\lambda)})$-unbiased, it now easily follows that $\E_{X\sim D_A}[|A_{i}\cdot X-w\cdot X|^2]\geq \exp({-O(\lambda)})\cdot\alpha^2$.  From \Cref{eqn:small_inner_product}, we have that that $|A_{ij}-w_{j}|=\alpha\leq \exp({O(\lambda)})\cdot \epsilon$. Setting the squared error to $\exp({-\Omega(\lambda)})\cdot \epsilon^2$ implies that $\norm{A_{i}-w}_{\infty}\leq \epsilon$.
\end{itemize}
\subsection{Learning in the High-Temperature Regime: the Analysis of Gaitonde and Mossel}
 Gaitonde and Mossel \cite{gaitonde2024unified} showed that the exponential dependence on $\lambda$ in the reduction from the previous subsection can be avoided when $A$ has iid Gaussian entries. First, they showed that as long as there is some constant $C$ such that  $\Pr_{X\sim D_A}[|A_i\cdot X|\leq C]\geq 3/4$, for any $w$ with small squared loss the following holds: $\E_{X\sim D_A}[|A_i\cdot X-w\cdot X|^2]\leq \exp(O(C))\cdot \epsilon^2$. The intuition is that in a region of constant mass, the anti-lipschitzness of the sigmoid scales much better than $\exp(-O(\lambda))$, namely it scales as $\exp(-O(C))$. 
More precisely, to improve step one from the previous subsection, they consider the conditional expectation over the event $\{A_{i}\cdot X\leq C\}$. Note that the expected squared error (after conditioning) is at most a constant factor off the true expectation. To avoid the exponential dependence on $\lambda$ from step two and recover the weight $A_{ij}$, they observe the following: if $D_A$ satisfied the additional property that $\Pr_{X\sim D_A}[|A_j\cdot X|\leq C]\geq 3/4$, then, the unbiasedness of the distribution conditioned on the event $\{\max(|A_i\cdot X|, |A_{j}\cdot X|)\leq C\}$ is effectively at least $\exp(-O(C))$. Thus, it suffices to set the squared error of Sparsitron to $\exp(-\Omega(C))\cdot \epsilon^2$. Formally, \cite{gaitonde2024unified} proved the following: 


\begin{theorem}(Theorem~6.1 from \cite{gaitonde2024unified})
\label{lem:gaitonde_recovery}
Let $A$ be an interaction matrix and $h\in \R^{n}$ be the external field of an Ising model. Suppose there exists a bound $C\geq 1$ such that the following hold:
(1) For each $i\in [j]$, we have $\norm{A_i}_{\infty}\leq C$, and (2) for each $i\in [n]$, it holds with probability at least $3/4$ over $X\sim D_{A,h}$ that $|A_i\cdot X+h_i|\leq C$.
Let $\widehat{A}$ be a matrix from $\R^{n\times n}$ and $\widehat{h}$ a vector from $\R^n$. Then, we have that for all $i\in [n]$, $j\neq i$ and $\Delta(X)\coloneq\bigr(\sigma(2(A_i\cdot X+h_i)-\sigma(2(\widehat{A}_i\cdot X+\widehat{h}_i))\bigr)^2$,
\begin{equation}
    \E_{X\sim D_{A,h}}[\Delta(X)]\geq \exp(-O(C))\cdot \min\{1,8(A_{ij}-\widehat{A}_{ij})^2\}.
\end{equation}Furthermore, if $\norm{A_i-\widehat{A}_i}_{\infty}\leq |h_i-\widehat{h}_i|/2n$, then it holds that $\E_{X\sim D_{A,h}}[\Delta(X)]\geq \exp(-O(C))\cdot \min\{1,8(h_i-\widehat{h}_i)^2\}.$

\end{theorem}

They proceed by showing that the conditions of the above lemma hold for the SK model for $\beta<1$. To do so, they use a recent result on the boundedness of the covariance matrix of the SK model at high temperatures \cite{el2024bounds,brennecke2023operator,brennecke2024two}. These works show that $\norm{\E_{X\sim D_A}[XX^T]}_{\text{op}}\leq \frac{2}{(1-\beta)^2}$ with probability at least $1-o(1)$ over the randomness in $A$. From the definition of the operator norm and an application of Chebyshev's inequality, this implies that for any unit vector $v$, $\Pr_{X\sim D_A}[|v\cdot X|\geq 4/(1-\beta)]\leq 1/4$. Now, applying this argument along the direction $A_i$ for any $i\in[n]$ implies the conditions required to apply \Cref{lem:gaitonde_recovery} with $C=4/(1-\beta)$. 

This gives a learning algorithm that has sample complexity and running time $\exp(O(1/(1-\beta)))\cdot \poly(n)$. It is known that boundedness of the covariance matrix experiences a phase transition at $\beta=1$ and becomes unbounded for $\beta>1$ \cite{gaitonde2024unified}. Thus, this approach will provably not work in the low-temperature regime. Furthermore, it is unclear if these covariance bounds hold for random Ising models on arbitrary graphs. We also note that the above argument does not hold for the case of non-zero external fields and requires significantly more technical work \cite{gaitonde2024unified}.

\subsection{Learning Random Ising Models Beyond the High-Temperature Regime}
In this section we prove our main result for learning random Ising models and obtain our claimed results for $\mathsf{SK}(\beta)$ as a corollary.  We begin with the following observation: for the task of parameter recovery, we do not need $\E_{X\sim D_{A,h}}[v\cdot X]$ to be small for all directions $v$, as is implied by a distribution with bounded covariance. Rather, we only care about the directions $\{A_i\}_{i\in [n]}$ corresponding to the rows of the interaction matrix as required by \Cref{lem:gaitonde_recovery}. We directly analyze the correlation of $D_{A,h}$ with these directions. We use the additional property that the rows of $A$ are subgaussian to show that with high probability over the choice of $A$, the random variable $A_i\cdot X+h_iX_i$ is subgaussian with subgaussianity constant $O(\beta^2+\beta\sqrt{\log n})$. 

To do this, we give a simple proof that there exists a universal constant $B$ such that $\E_{X\sim D_{A,h}}[\exp(|A_i\cdot X+h_iX_i|^2/B\beta^2)]\leq \exp(\beta^2)\cdot O(n^2)$ with probability at least $1-O(1/n)$ over the random choice of $A$ and $h$ (\Cref{claim:MGF_high_probability}). Having shown this, a tail bound on $|A_i\cdot X+h_iX_i|$ follows immediately by using the standard trick of exponentiating and applying Markov's inequality (\Cref{lem:high_probability_small_correlation}). 
Our main lemma is as follows:
  \begin{lemma}
        \label{claim:MGF_high_probability}
        With probability $1-\frac{1}{n}$ over $D_{A,h}\sim \dpsi_{G,\beta}$, there exists a universal constant $B$ for all $i\in [n]$, we have 
        \begin{equation}
            \E_{X\sim  D_{A,h}}\bigr[\exp\bigr({|A_i\cdot X+h_iX_i|^2/(B\beta^2)}\bigr)\bigr]\leq n^2 \exp(O(\beta^2)).
        \end{equation}
    \end{lemma}
    \begin{proof}
        We compute $\E_{A}\E_{X\sim D_A}\bigr[\exp\bigr(|A_i\cdot X+h_iX_i|^2/B\beta^2\bigr)\bigr]$ for all $i \in [n]$. We choose the universal constant $B$ later in the proof. Let $(A,h)_{-i}$ denote the entries of $(A,h)$ that do not involve the variable $i$. These are independent from $(A_i,h_i)$. 
        \begin{align}
        \label{eqn: mgf_align_1}
            \E_{A,h}&\E_{X\sim D_{A,h}}\bigr[\exp\bigr(|A_i\cdot X+h_iX_i|^2/B\beta^2\bigr)\bigr]=\E_{(A,h)_{-i}}\E_{A_i,h_i}\E_{X\sim D_{A,h}}\bigr[\exp\bigr(|A_i\cdot X+h_iX_i|^2/B\beta^2\bigr)\bigr]\nonumber\\
            &=\E_{(A,h)_{-i}}\E_{A_i,h_i}\Bigr[\sum_{x\in \cube{n}}\Pr_{X\sim D_{A,h}}[X=x]\cdot \mgfi\Bigr]\nonumber\\
            &=\E_{(A,h)_{-i}}\E_{A_i,h_i}\Bigr[\sum_{x\in \cube{n}}\frac{\exp(\sum_{j< k}A_{ij}x_i x_j+\sum_{j=1}^{n}h_jx_j)}{Z(A,h)} \cdot \mgfi\Bigr]\nonumber\\
            &=\E_{(A,h)_{-i}}\E_{A_i,h_i}\Bigr[\sum_{x\in \cube{n}}\frac{\exp(x_i\sum_{j\neq i}A_{ij}x_j+h_ix_i+g_{(A,h)_{-i}}(x))}{Z(A,h)}\cdot \mgfi\Bigr]\nonumber\\
            &=\E_{(A,h)_{-i}}\Bigr[\sum_{x\in \cube{n}}\exp(g_{(A,h)_{-i}}(x))\E_{A_i,h_i}\bigr[\frac{\exp(x_i(A_i\cdot x)+h_ix_i)}{Z(A,h)}\cdot\mgfi\bigr] \Bigr]
        \end{align}

        where $g_{(A,h)_{-i}}(x)=\exp(\sum_{j< k \text{,} j\neq i\neq k}A_{jk}x_jx_k)+\sum_{j\neq i}h_jx_j$ only depends on $(A,h)_{-i}$ and is independent of $A_i,h_i$. To further bound the right hand side above, we need a lower bound on $Z(A,h)$, the partition function. We do so by marginalizing out the dependence of $Z(A,h)$ on $A_i,h_i$ so that we can remove this term from the expectation over $A_i,h_i$. We have that 
        \begin{align}
        \label{eqn:mgf_align_2}
            Z(A,h)&=\sum_{x\in \cube{n}}\exp(\sum_{j< k}A_{jk}x_jx_k+\sum_{i=1}^{n}h_ix_i)
            =\sum_{x\in \cube{n}}\exp(x_i(A_i\cdot x)+h_ix_i)\exp(g_{(A,h)_{-i}}(x))\nonumber\\
            &= \frac{1}{2}\sum_{x\in \cube{n}}\bigr(\exp(A_i\cdot x+h_i)+\exp(-A_i\cdot x-h_i)\bigr)\exp(g_{(A,h)_{-i}}(x))\nonumber\\
           &\geq\frac{1}{2}\sum_{x\in \cube{n}}\exp(g_{(A,h)_{-i}}(x))
        \end{align}
        where the third equality follows from the fact that $A_i\cdot x+h_i$ and $g_{(A,h)_{-i}}(x)$ do not depend on $x_i$. The final inequality follows from the fact that $e^t+e^{-t}\geq 1$ for any $t\in \R$. 
        Combining the \Cref{eqn: mgf_align_1,eqn:mgf_align_2}, we obtain that 
        \begin{align}
            \label{eq:MGF1}
            \E_{A,h}&\E_{X\sim  D_{A,h}}\bigr[\exp\bigr({|A_i\cdot X+h_iX_i|^2/4\beta^2}\bigr)\bigr]\nonumber\\&=\E_{(A,h)_{-i}}\Bigr[\sum_{x\in \cube{n}}\exp(g_{(A,h)_{-i}}(x))\E_{A_i,h_i}\bigr[\frac{\exp(x_i(A_i\cdot x)+h_ix_i)}{Z(A,h)}\cdot\mgfi\bigr] \Bigr]\nonumber\\
            &\leq 2\E_{(A,h)_{-i}}\Bigr[\sum_{x\in \cube{n}}\exp(g_{(A,h)_{-i}}(x))\E_{A_i,h_i}\bigr[\frac{\exp(x_i(A_i\cdot x+h_i)+(A_i\cdot x+h_ix_i)^2/B\beta^2)}{\sum_{x\in \cube{n}}\exp(g_{(A,h)_{-i}}(x))}\bigr] \Bigr]\nonumber\\
            &=2\cdot \E_{(A,h)_{-i}}\Bigr[\sum_{x\in \cube{n}}\frac{\exp(g_{(A,h)_{-i}}(x))\E_{A_i,h_i}\bigr[{\exp(x_i(A_i\cdot x+h_i)+|A_i\cdot x+h_ix_i|^2/B\beta^2)}\bigr]}{{\sum_{x\in \cube{n}}\exp(g_{(A,h)_{-i}}(x))}} \Bigr]
        \end{align}
        To bound the above quantity, it suffices to bound $\E_{A_i,h_i}\bigr[{\exp(x_i(A_i\cdot x+h_i))}\cdot\mgfi\bigr]$ for arbitrary $x$. Since $A_i$ is a vector of $d$ independent $\frac{\beta}{\sqrt{d}}$ subgaussian random variables and $h_i$ is $O(1)$-subgaussian, the distribution of $A_i\cdot x+h_i$ is $O(\beta)$-subgaussian for fixed $x$. This is because the sum of $d$ $\lambda$-subgaussian random variables is $\lambda\sqrt{d}$-subgaussian. Thus, we have 
        \begin{align}
        \label{eq:MGF2}
            E_{A_i,h_i}&\bigr[{\exp(x_i(A_i\cdot x)+h_ix_i)}\cdot\mgfi\bigr]\nonumber\\
            &\leq \sqrt{\E_{A_i,h_i}[\exp(2x_i(A_i\cdot x+h_i))]\cdot \E_{A_i,h_i}[\exp(2(A_i\cdot x+h_ix_i)^2/B\beta^2)]}\leq \exp(O(\beta^2))
        \end{align}
        where the first inequality follows from Cauchy Schwarz inequality and the last inequality follows from \Cref{fact:subgaussian} when $B$ is an appropriately chosen universal constant.
        Combining \Cref{eq:MGF1,eq:MGF2}, we obtain that  $\E_{A,h}\E_{X\sim  D_{A,h}}\bigr[\exp\bigr({|A_i\cdot X+h_iX_i|^2/B\beta^2}\bigr)\bigr]\leq  \exp(O(\beta^2))$. 

        Now, applying Markov's inequality implies that with probability at least $1-\frac{1}{n^2}$ over $A$, for a fixed $i\in [n]$, we have that $\E_{X\sim D_{A,h}}\bigr[\exp\bigr({|A_i\cdot X+h_iX_i|^2/B\beta^2}\bigr)\bigr]\leq n^2\exp(O(\beta^2))$. A union bound over $i\in [n]$ completes the proof.
    \end{proof}
Having proved the above lemma, the following tail bound on $A_i\cdot X+h_iX_i$ immediately follows.
\begin{lemma}
\label{lem:high_probability_small_correlation}
   Let $G$ be a graph of degree $d$ and let $\beta>0$. With probability $1-\frac{1}{n}$ over $D_{A,h}\sim \dpsi_{G,\beta}$ with interaction matrix $A$ and gaussian external field, for all $i\in [n]$, we have 
 \begin{equation}
    \label{eq:sk_concentration_2}
    \Pr_{X\sim D_{A,h}}[|A_i\cdot X+h_iX_i|\leq O(\beta^2+\beta\sqrt{\log n})]\geq 1-\frac{1}{n}
\end{equation}
\end{lemma} 
\begin{proof}
   The proof follows almost immediately from \Cref{claim:MGF_high_probability}. Observe that with probability $1-\frac{1}{n}$ over $D_{A,h}\sim \dpsi_{G,\beta}$, for any $i\in [n]$, we have that 
    \begin{align*}
        \Pr_{X\sim D_{A,h}}[&|A_i\cdot X+h_iX_i|\geq t]=\Pr_{X\sim D_{A,h}}\Bigr[\exp\bigr(\frac{|A_i\cdot X+h_iX_i|^2}{B\beta^2}\bigr)\geq \exp\bigr(\frac{t^2}{B\beta^2}\bigr)\Bigr]\\
        &\leq \frac{\E_{X\sim D_{A,h}}\Bigr[\exp\bigr({|A_i\cdot X+h_iX_i|^2/B\beta^2}\bigr)\Bigr]}{\exp\bigr(\frac{t^2}{B\beta^2}\bigr)}
        \leq n^2\exp(O(\beta^2))\exp\bigr(-(t/B\beta)^2\bigr)
    \end{align*}
    where $B$ is the universal constant from \Cref{claim:MGF_high_probability}. The first inequality follows from Markov's and the final inequality follows from \Cref{claim:MGF_high_probability}. Setting $t=O(\beta^2+\beta\sqrt{\log n})$ makes the above probability less than $\frac{1}{n}$. 
\end{proof}
Our theorem on parameter recovery now follows. 
\begin{theorem}
\label{thm:random_ising_param}
     Let $G$ be a graph of degree $d$ and $0<\epsilon,\delta\leq 1$. With probability at least $1-O(1/n)$ over $D_{A,h}\sim D_{G,\beta}$, there exists an algorithm that draws $N=\exp(O(\beta^2+\beta\sqrt{\log n}))\cdot {O}\bigr(\frac{\beta^2 d \log n \log (n/\delta \epsilon)}{\epsilon^4}\bigr)$ samples and runs in time $O(n^2\cdot N)$ that outputs a matrix $\widehat{A}$ such that $\norm{A-\widehat{A}}_{\infty}\leq \epsilon$ and $\norm{h-\widehat{h}}_\infty \leq n\epsilon$. The algorithm succeeds with probability $1-\delta$. 
\end{theorem}
\begin{proof}
   Observe that the interaction matrix $A$ and external field $h$ for a random Ising model satisfies $\norm{A}_{\infty}\leq O\big(\frac{\beta}{\sqrt{d}}\sqrt{\log n}\big)$ and $\norm{h}_{\infty}\leq O(\sqrt{\log n})$ with high probability by applying standard subgaussian concentration. Now, 
   Applying \Cref{lem:gaitonde_recovery} with $C=O(\beta^2+\beta\sqrt{\log n})$ and \Cref{thm:sparsitron} with error $\epsilon=\exp(-O(\beta^2+\beta\sqrt{\log n}))\cdot \epsilon^2$ and $\lambda=O(\beta\sqrt{d\log n})$, we obtain the theorem.
\end{proof}
Our structure learning result immediately follows from the above theorem. 
From a standard argument that parameter recovery implies closeness in TV distance (\Cref{lem:param_rec_tv_distance}), we immediately obtain the following corollary by setting the error to $\epsilon/n^2$.
\begin{corollary}
\label{clry:random_ising_tv}
    Let $G$ be a graph of degree $d$ and $0<\epsilon,\delta\leq 1$. With probability at least $1-O(1/n)$ over $D_{A,h}\sim D_{G,\beta}$, there exists an algorithm that draws $N=\exp(O(\beta^2+\beta\sqrt{\log n}))\cdot {O}\left(\frac{\beta^2 n^8d \log n \log (n/\delta \epsilon)}{\epsilon^8}\right)$ samples and runs in time $O(n^2\cdot N)$ that outputs a matrix $\widehat{A}$ and vector $\widehat{h}$ such that (1) $\kl(D_{A,h},D_{\widehat{A},\widehat{h}})\leq 2\epsilon^2$ and (2) $\tv(D_{A,h},D_{\widehat{A},\widehat{h}})\leq \epsilon$. The algorithm succeeds with probability $1-\delta$. 
\end{corollary}

We instantiate the above two statements to the SK model (complete graph with Gaussian weights) to obtain the following corollaries.
\begin{theorem}
\label{thm:sk_param}
    Let $\beta>0$ and $0<\epsilon,\delta\leq 1$. With probability at least $1-O(1/n)$ over $D_{A,h}\sim \mathsf{SK}(\beta)$, there exists an algorithm that draws $N=\exp(O(\beta^2+\beta\sqrt{\log n}))\cdot {O}\left(\frac{\beta^2 n \log n \log (n/\delta \epsilon)}{\epsilon^4}\right)$ samples and runs in time $O(n^2\cdot N)$ that outputs a matrix $\widehat{A}$ such that $\norm{A-\widehat{A}}_{\infty}\leq \epsilon$ and $\norm{h-\widehat{h}}_\infty \leq n\epsilon$.
\end{theorem}

\begin{corollary}
\label{clry:sk_tv}
   Let $\beta>0$ and $0<\epsilon,\delta\leq 1$. With probability at least $1-O(1/n)$ over $D_{A,h}\sim \mathsf{SK}(\beta)$, there exists an algorithm that draws $N=\exp(O(\beta^2+\beta\sqrt{\log n}))\cdot {O}\left(\frac{\beta^2 n^9 \log n \log (n/\delta \epsilon)}{\epsilon^8}\right)$ samples and runs in time $O(n^2\cdot N)$ that outputs a matrix $\widehat{A}$ and vector $\widehat{h}$ such that (1) $\kl(D_{A,h},D_{\widehat{A},\widehat{h}})\leq 2\epsilon^2$ and (2) $\tv(D_{A,h},D_{\widehat{A},\widehat{h}})\leq \epsilon$.
\end{corollary}
In the case where the matrix $A$ and external field are rademacher random variables, we show that we can exactly recover the model using only a sub-polynomial number of samples. This is in contrast to the Gaussian case where our sample complexity was polynomial in $n$, and we only recovered the model approximately.
\begin{theorem}
\label{thm:random_ising_param_rademacher}
   Let $G$ be a graph of degree $d$ and $0<\epsilon,\delta\leq 1$. With probability at least $1-O(1/n)$ over $D_{A,h}\sim D_{G,\beta}$ (with rademacher weights), there exists an algorithm that draws $N=\exp(O(\beta^2+\beta\sqrt{\log n}))\cdot {O}\left(\frac{d \log (n/\delta \beta)}{\beta^2}\right)$ samples and runs in time $O(n^2\cdot N)$ that recovers the distribution $D_{A,h}$ exactly. 
\end{theorem}
This is a special case of a more general theorem (\Cref{thm:random_mrf_param_rademacher}). We refer to \Cref{sec:random_mrf_rademacher} for the proof.
\begin{remark}
We note that the analysis of \Cref{lem:high_probability_small_correlation}  and hence all the learning results straightforwardly extend to the case where $h_{i}\sim \Gauss(\mu,\sigma)$. In this case, for $\gamma=\sqrt{\beta^t+\sigma^2}$, we obtain that $|A_i\cdot X+h_iX_i|$ is at most $O(\mu+\gamma^2+\gamma\sqrt{\log n})$ with probability at least $1-O(1/n)$. This implies sample complexity and running time that scale with $\exp(O(\mu+\gamma^2+\gamma\sqrt{\log n}))\cdot \poly(n)$.
\end{remark}
\begin{remark}
    Note that most of the analysis above would hold even if we had used a simpler algorithm for learning a sigmoid with respect to square loss such as GLMTron \cite{kakade2011efficient}. GLMtron, however, has polynomial sample complexity (in $n$) when the input norm is $\sqrt{n}$. On the other hand, Sparsitron  has sample complexity $O(\lambda^2\cdot \log n)$. This improved dependence on $n$ is crucial to obtain our sub-polynomial sample complexity bound in  \Cref{thm:random_ising_param} and \Cref{thm:random_ising_param_rademacher}.
\end{remark}
\section{Learning Random MRFs}
We now present the proofs of our result on learning random $t$-MRFs. First we sketch the analysis. Consider an MRF $\dmrf$. Note that $\Pr_{X\sim \dmrf}[X_i=1\mid X_{[n]\setminus \{i\}}=x]=\sigmoid(2\di\psi(x))$.  We run Sparsitron nodewise to obtain a polynomial $p$ that attains squared error (for learning the sigmoid) of at most $\epsilon^2$ with respect to $\di\psi$. In \cite{sparsitron}, they argue that an MRF with width at most $\lambda$ is $\exp(-O(\lambda))$-unbiased. They use this to show that squared loss of at most $\epsilon^2$ implies that the maximal monomials of $\di\psi$ and $p$ are $\exp(O(\lambda t)) \epsilon$ close. Recovering the maximal monomials is sufficient for structure recovery.  As in \Cref{sec:ising}, we want to avoid worst-case bounds (i.e., naively setting $\lambda=\Omega(n^{(t-1)/2})$) by conditioning on some constant probability regions.

We condition on the event $\cevent\coloneq \{x\mid |\psi(x)-\psi(y)|\leq C\text{ for all $y$ with $\dhamming(x,y)\leq t$}\}$ (see \Cref{defn:csmooth}). We argue that $\dmrf$ conditioned on $\cevent$ behaves similarly to an $\exp(-O(C))$-unbiased distribution and hence obtain parameter recovery without an exponential dependence on $\lambda$ (\Cref{lem:struct_recovery_polynomial,lem:polynomial_recovery}). These results generalize \Cref{lem:gaitonde_recovery} to the case of MRFs. Finally, we argue via a generalization of \Cref{claim:MGF_high_probability,lem:high_probability_small_correlation} that for any random MRF, the event $\cevent$ occurs with constant probability (\Cref{thm:subgauss_to_smooth}). In doing so, we bypass the need for boundedness of moment matrices as conjectured by \cite{gaitonde2024unified}.
\subsection{Subgaussian Derivatives, $C$-smooth MRFs and their properties}
We now define the notion of an MRF with subgaussian derivatives. This is a property that holds true for all the random MRF families that we study in this work. We argue that for any distribution over MRFs $\dpsi$ with subgaussian derivative,  $\dmrf\sim \dpsi$  satisfies a deterministic condition (\Cref{defn:csmooth}) with high probability that implies efficient structure learning and parameter recovery by running Sparsitron. 

\begin{definition}[Subgaussian Derivatives]
\label{def:subgauss_derivatives}
    Let $\dpsi$ be a distribution over factorization polynomials $\psi$ of degree $t$ such that each coefficient of $\psi\sim \dpsi$ is independently picked. Let $\lambda>0$. We say that $\dpsi$ has $\lambda$-subgaussian derivatives if it holds that 
    \begin{enumerate}
        \item For all $x,y\in \cube{n}$ such that $\dhamming(x,y)\leq t$, we have that the random variable $\left(\psi(x)-\psi(y)\right)$ is $\lambda $-subgaussian where the randomness is over $\psi\sim \dpsi$.
    \end{enumerate}
\end{definition}
For any vector $x\in \cube{n}$ and set $S\subseteq[n]$, we denote the vector obtained by flipping the coordinates of $x$ in the set $S$ by $x^{S}$. Let the function $\psi^S$ be defined as 
 \[
 \psi^{S}(x)\coloneq \psi(x)-\psi(x^{S})=2\sum_{|T\cap S|\text{ is odd }}\widehat{\psi}(T)\cdot \parity{T}(x).    
 \] 
In particular, we have that $\psi^{\{i\}}(x)=2x_i\cdot \di\psi(x)$.
We also refer to $\dpsi$ as a distribution over MRFs because the factorization polynomial uniquely determines the MRF.  We now define the deterministic conditions on $\dmrf\sim \dpsi$ that imply efficient learning of $\dmrf$. 
\begin{definition}[\csmooth MRF]
    \label{defn:csmooth}
    Let $\dmrf$ be a $t$-wise MRF with factorization polynomial $\psi$. Let $\cevent\subseteq \cube{n}$ be the set defined as 
    \[
    \cevent\coloneq \{x\mid |\psi(x)-\psi(y)|\leq C\text{, for all $y$ such that } \dhamming(x,y)\leq t\}    
    \]
   We say that $\dmrf$ is \csmooth if $\Pr_{X\sim D}[X\in \cevent]\geq \frac{7}{8}$.
\end{definition}

To motivate the above definition and why it implies efficient learning, first consider the case of the Ising model $D_{A,h}$. In this case, the above definition corresponds to the property that with constant probability over $X\sim D_{A,h}$, we have that $|A_i\cdot x+h_i|\leq C$ for all $i\in [n]$. This was exactly the property used in the case of learning random Ising models in \Cref{sec:ising}. 

In the case of $t$-MRFs, the analysis of \cite{sparsitron} used the 
worst case property of $\delta$-unbiasedness of these distributions. They proved that for any distribution $D$ that is $\delta$-unbiased, given a polynomial $p$ of degree at most $t$ and maximal monomial $S$, it holds that $\Pr_{X\sim D}[|p(X)|\geq |\widehat{p}(S)|\geq \delta^{t}$. Using this property and the fact that the width of the model is bounded by $\lambda$ , they showed how to learn the coefficients of $p$ from samples. Note that $\delta\geq \exp(-\Omega(\lambda))$. Since $\delta$ scales with the width, we cannot directly use their analysis directly for learning the polynomial without paying $\exp(O(\lambda))$. Here is where the above condition of $C$-smoothness helps us. It is easy to argue that for any polynomial  $p$ and point $x$, there exists a $y$ at Hamming distance at most $t-1$ from $x$ such that $|p(y)|\geq |\widehat{p}(S)|$ (\Cref{lem:anticonc_dist}). Now, using a simple argument, we see that any \csmooth MRF also has a similar anticoncentration property as the one required by \cite{sparsitron}. 

\begin{lemma}[Anticoncentration of \csmooth $\dmrf$]
    Let $\dmrf$ be a \csmooth $t$-MRF. Then for any polynomial $p$ of degree $t$ with maximal monomial $S$, it holds that $\Pr_{X\sim \dmrf}[|p(X)|\geq |\widehat{p}(S)|]\geq 2^{-(t+1)}\cdot \exp(-2C)$. 
\end{lemma}
\begin{proof}
For any $x\in \cevent$ and $y,z$ that differ from $x$ only in the coordinates in $S$, it holds that $|\psi(y)-\psi(z)|\leq 2C$. Let $z$ maximize the quantity $\Pr_{X\sim \dmrf}[X=z\mid X_{[n]\setminus S}=x]$. Thus, for any $y$ with $y_{[n]\setminus S}=x_{[n]\setminus S}$, it holds that 
\[
\Pr_{X\sim \dmrf}[X=y\mid X_{[n]\setminus s}=x_{[n]\setminus S}]\geq 2^{-t}\cdot \frac{\Pr_{X\sim D_{\psi}}[X=y]}{\Pr_{X\sim D_{\psi}}[X=z]}\geq 2^{-t}\exp(-2C)
\] where the first inequality holds from the definition of $z$ and the fact that at least one element has conditional probability greater than $2^{-t}$. Let $\scevent$ be the set $\{y\mid x_{[n]\setminus S}=y,x\in \cevent\}$. Thus, we have 
\begin{align*}\Pr_{X\sim \dmrf}[|p(X)|&\geq |\widehat{p}(S)|]\geq \sum_{w\in \scevent}\Pr_{X\sim \dmrf}[|p(X)|\geq |\widehat{p}(S)|\mid X_{[n]\setminus S}=w]\cdot \Pr_{X\sim \dmrf}[X_{[n]\setminus S}=w]\\
&\geq 2^{-t}\exp(-2C)\Pr_{X\sim \dmrf}[X\in \scevent]\geq 2^{-(t+1)}\cdot\exp(-2C)\end{align*}
where the penultimate inequality follows from the fact that there exists some string $z$ that differs from $x$ only in $S$ such that $|p(z)|\geq |\widehat{p}(S)|$. The second inequality follows from the previous argument that the probability of $X=z$ conditioned on $X_{[n]\setminus S}=w$ is at least $2^{-t}\exp(-2C)$. The last inequality follows from the fact that $\Pr_{X\sim \dmrf}[X\sim \scevent]\geq \Pr_{X\sim \dmrf }[X\in \cevent]\geq \frac{7}{8}$.
\end{proof}
Although we do not use the above result directly in our proofs, we use similar ideas to extend the analysis of \cite{sparsitron} to our setting.

We now prove a structural result that says that any MRF with subgaussian derivatives is also \csmooth for appropriate choice of $C$ (a generalization of \Cref{lem:high_probability_small_correlation}). This result contains most of the technical novelty of our work, and we believe it could be useful in other contexts as well.
\begin{theorem}
\label{thm:subgauss_to_smooth}
    Let $\dpsi$ be a distribution over $t$-wise MRFs such that $\dpsi$ has $\lambda$-subgaussian derivatives. Then, with probability at least $1-\frac{1}{(n+1)^t}$ over $\dmrf\sim \dpsi$, we have that $\dmrf$ is $O(\lambda^2+\lambda \sqrt{t\log n})$-smooth.
\end{theorem}
\begin{proof}
 Let the function $\psi^{-S}$ be the polynomial containing the coefficients of $\psi$ not in $\psi^S$. That is, $\psi^{-S}(x)\coloneq \sum_{|T\cap S|\text{ is even }}\widehat{\psi}_T\cdot \parity{T}(x)$. Clearly, we have that $\psi(x)=\psi^{-S}(x)+\frac{\psi^S(x)}{2}$. We also have that $\psi(x^S)=\psi^{-S}(x)-\frac{\psi^S(x)}{2}$. 

 We are now ready to start the proof. The result follows by applying Markov's inequality to the following claim. 
 \begin{claim}
    \label{claim:mrf_mgf}
    With probability $1-\frac{1}{(n+1)^t}$ over $\dmrf\sim \dpsi$, for all $S\subseteq [n]$ with $|S|\leq t$, we have 
    \begin{equation}
        \E_{X\sim \dmrf}[\exp\left((\psi^S(X))^2/B\right)]\leq 4(n+1)^{2t}\exp(B)
    \end{equation} where $B=2(C\lambda )^2$ for large universal constant $C$.
 \end{claim}
\begin{proof}
    We compute the quantity $\E_{\dmrf\sim \dpsi}\E_{X\sim \dmrf}\bigr[\exp\left((\psi^{S}(X))^2/B\right)\bigr]$ for all sets $S$ with size at most $t$. We have that 
    \begin{align}
        \label{eqn:mrf_mgf_1}
        \E_{\dmrf\sim \dpsi}\E_{X\sim \dmrf}\bigr[\exp\left((\psi^{S}(X))^2/B\right)\bigr]&=\E_{\psi^{-S}}\E_{\psi^S}\E_{X\sim D_\psi}\bigr[\exp\left((\psi^S(X))^2/B\right)\bigr]\nonumber\\
        &=\E_{\psi^{-S}}\E_{\psi^S}\bigr[\sum_{x\in \cube{n}}\Pr_{X\sim D_{\psi}}[X=x]\cdot\exp\left((\psi^S(x))^2/B\right)\bigr]\nonumber\\
        &=\E_{\psi^{-S}}\E_{\psi^S}\biggr[\sum_{x\in \cube{n}}\frac{\exp(\psi(x))}{Z_{\psi}}\cdot\exp\left((\psi^S(x))^2/B\right)\biggr]
    \end{align}
    where $Z_{\psi}$ is the partition function of the MRF defined as $Z_{\psi}=\sum_{x\in \cube{n}}\exp(\psi(x))$. To proceed in bounding the quantity in \Cref{eqn:mrf_mgf_1}, we first need to decouple $Z_{\psi}$ from $\psi^{S}$ as this quantity is in the denominator and hence hard to analyze. To do this, we lower bound $Z_{\psi}$ using only the polynomial $\psi^{-S}$. We have that 
    \begin{align}
        \label{eqn:mgf_partition}
        Z_{\psi}&=\sum_{x\in \cube{n}}\exp(\psi(x))\nonumber
        =\frac{1}{2}\sum_{x\in \cube{n}}\bigr(\exp(\psi(x))+\exp(\psi(x^S))\bigr)\nonumber\\
        &=\frac{1}{2}\sum_{x\in \cube{n}}\Bigr(\exp\bigr(\psi^{-S}(x)+\frac{\psi^S(x)}{2}\bigr)+\exp\bigr(\psi^{-S}(x)-\frac{\psi^S(x)}{2}\bigr)\Bigr)\nonumber
        \\&\geq \frac{1}{2}\sum_{x\in \cube{n}}\exp\left(\psi^{-S}(x)\right)
    \end{align} where we obtain the second equality by pairing terms that are equal to each other outside $S$ and the complement of each other in $S$. The third equality follows from the definition of the polynomials $\psi^S$ and $\psi^{-S}$ and the final inequality follows from the fact that $e^{t}+e^{-t}\geq 1$ for all $t\in \R$. Combining \Cref{eqn:mrf_mgf_1,eqn:mgf_partition}, we obtain that 
    \begin{align*}
        \E_{\dmrf\sim \dpsi}&\E_{X\sim \dmrf}\left[\exp\left((\psi^{S}(X))^2/B\right)\right]\\&\leq 2\cdot \E_{\psi^{-S}}\E_{\psi^S}\left[\sum_{x\in \cube{n}}\frac{\exp(\psi(x))}{\sum_{x\in \cube{n}}\exp\left(\psi^{-S}(x)\right)}\cdot\exp\left((\psi^S(x))^2/B\right)\right]\\
        &= 2\cdot \E_{\psi^{-S}}\E_{\psi^S}\left[\sum_{x\in \cube{n}}\frac{\exp\left(\psi^{-S}(x)+\frac{\psi^S(x)}{2}\right)}{\sum_{x\in \cube{n}}\exp\left(\psi^{-S}(x)\right)}\cdot\exp\left((\psi^S(x))^2/B\right)\right]\\
        &=2\cdot \E_{\psi^{-S}}\left[\sum_{x\in \cube{n}}\frac{\exp(\psi^{-S}(x))\E_{\psi^S}\left[\exp\left(\frac{\psi^{S}(x)}{2}\right)\cdot \exp\left(\left(\psi^S(x)\right)^2/B\right)\right]}{\sum_{x\in \cube{n}}\exp(\psi^{-S}(x))}\right]
    \end{align*} where the penultimate equality follows from the fact that $\psi(x)=\psi^{-S}(x)+\frac{\psi^S(x)}{2}$ and the last equality follows linearity of expectation. Clearly, to bound the RHS in the above argument, it suffices to bound $\E_{\psi^S}\left[\exp\left(\frac{\psi^{S}(x)}{2}\right)\cdot \exp\left(\left(\psi^S(x)\right)^2/B\right)\right]$ pointwise for any $x\in \cube{n}$. Note that from the assumption of $\lambda$-subgaussian derivatives, we have that the random variable $\psi^S(x)$ is $\lambda$-subgaussian. Let $Y$ denote the random variable $\psi^S(x)$. We want to bound $\E_{Y}[\exp(Y/2)\cdot \exp(Y^2/(2\cdot(C\lambda)^2))]$ given that $Y$ is $\lambda$-subgaussian. We have that 
    \begin{align*}
        \E_{Y}\left[\exp(Y/2)\cdot \exp(Y^2/(2\cdot(C\lambda)^2))\right]&\leq \sqrt{\E_{Y}\left[\exp(Y)\right]\cdot \E_{Y}\left[\exp(Y^2/(C\lambda )^2)\right]}
        \leq 2\exp(C^2\lambda^2)
    \end{align*}
    where the first inequality is Cauchy-Schwarz and the last inequality follows from \Cref{fact:subgaussian}. Now, combining the above arguments, above gives us $\E_{\dmrf\sim \dpsi}\E_{X\sim \dmrf}\left[\exp\left((\psi^{S}(X))^2/B\right)\right]\leq 4\exp(C^2\lambda^2)=4\exp(B)$. Thus, for any fixed $S$ of size at most $t$, Markov's inequality implies that \[\Pr_{\dmrf\sim \dpsi}\left[\E_{X\sim \dmrf}\left[\exp\left((\psi^{S}(X))^2/B\right)\right]\geq 4(n+1)^{2t}\exp(B)\right]\leq \frac{1}{(n+1)^{2t}}.\]
    Since the number of sets of size at most $t$ is bounded above by $(n+1)^t$, a union bound implies the claim. 
\end{proof}
We are now ready to complete the proof of the theorem. Let $\dmrf\sim \dpsi$ be an MRF for which the event from \Cref{claim:mrf_mgf} holds true. This happens with probability at least $1-\frac{1}{(n+1)^t}$. For any set $|S|\leq t$, we have that
\begin{align*}
    \Pr_{X\sim \dmrf}\left[|\psi^S(X)|\geq r\right]&=\Pr_{X\sim \dmrf}\left[(\psi^S(X))^2/B\geq r^2/B\right]
    \leq \E_{X\sim \dmrf}\left[\exp\left((\psi^{S}(X))^2/B\right)\right]\cdot \exp(-r^2/B)\\
    &\leq 4(n+1)^{2t}\exp(B)\exp(-r^2/B)\leq \frac{1}{8(n+1)^t}
\end{align*}
where $B=2(C\lambda)^2$. The second inequality follows by taking the exponent on both sides and applying Markov's inequality. The third inequality follows from \Cref{claim:mrf_mgf}. The final inequality follows by settings $r=O(B+\sqrt{B t\log n})$. Now, a union bound over all sets of size at most $t$ completes the proof. 
\end{proof}

\subsection{Results on Structure Learning and Parameter Recovery}

We now state our results on structure learning and parameter recovery for \csmooth MRFs. We need the following standard non-degeneracy condition (introduced in \cite{sparsitron}) that is required to ensure identifiability of the underlying graph given samples from the MRF. 
\begin{definition}[$\eta$-identifiability]
    A $t$-MRF with factorization polynomial $\psi$ with dependency graph $G$ is said to be $\eta$-identifiable for $\eta>0$ if $|\widehat{\psi}(S)|\geq \eta$ for all maximal monomials $S$ in $\psi$ and all edges in $G$ are contained in a monomial of $\psi$.
\end{definition}

We are now ready to state and prove the theorem on structure learning in \csmooth MRFs which are $\eta$-identifiable. 
\begin{theorem}
  \label{thm:structure learning MRFs}
  Let $C,\lambda,\eta>0$. Let $\dmrf$ be a $C$-smooth $t$-MRF with dependency graph $G$ that is $\eta$-identifiable and suppose $
  \norm{\di \psi}_1\leq \lambda$ for all $i\in [n]$. Then, there exists an algorithm that draws $N=O\left(\frac{\lambda^22^{O(t)}\exp(O(C))}{\eta^4}\log(n/\delta\eta)\right)$ samples from $\dmrf$ and runs in time equal $O\left(N\cdot n^t\right)$ such that it finds the graph $G$ with probability at least $1-\delta$.
\end{theorem}
\begin{proof}
    The algorithm is exactly the same as Algorithm~3 in \cite{sparsitron} with an appropriate choice of parameters. We only sketch a proof here since our analysis is almost identical to that of \cite{sparsitron} except that we use \Cref{lem:struct_recovery_polynomial} instead of Lemma~6.2 from \cite{sparsitron}. First, for each $i\in [n]$, we obtain polynomials $p_i$ such that \begin{equation}
    \label{eqn:sigmoid_rec}
    \E_{X\sim \dmrf}\left[(\sigmoid(p_i(X))-\sigmoid(2\di\psi(X)))^2\right]\leq \frac{\eta^2}{2^{t+4}\exp(10C+6)}\end{equation}
     To do this, we use the property that $\Pr[X_i=1\mid X_{[n]\setminus\{i\}}]=\sigmoid(2\di\psi(X))$ and run the algorithm from \Cref{thm:sparsitron} after doing a feature expansion of $X$ to the monomial basis containing all monomials of degree less than $t-1$. We run Sparsitron such for each $i$ the success probability is $O(\delta/n)$. Thus, with probability $O(\delta)$, we have that \Cref{eqn:sigmoid_rec} holds for all $i\in [n]$. The sample complexity so far was $\frac{\lambda^2\cdot 2^{O(t)}\exp(O(C))}{\eta^4}\log(n/\delta\eta)$.
     
     Now, since $\dmrf$ is $C$-smooth, we have from \Cref{lem:struct_recovery_polynomial} that 
     \[
     \Pr_{X\sim \dmrf}\left[|\widehat{2\di\psi}({S})-\partial_{S}p_i(X)|>\eta/2|\right]\leq \frac{1}{4}.
     \]
    The rest of the analysis is exactly the same as \cite{sparsitron}. We construct an output graph $H$ iteratively. Draw $K=O(\log(n^t/\delta))$ independent samples from $\dmrf$. For each $i\in [n]$, and $S\subseteq [t-1]$, evaluate $\frac{\partial_Sp_i(X)}{2}$ on each of these samples. If the median of these $K$ evaluations is greater than $\eta/2$, then add the clique on the vertices in $S\cup\{i\}$ to the graph $H$. Using the concentration of median, $\eta$-identifiability of $G$ and a union bound over all  monomials of degree less than $t$ and vertices $i\in [n]$, we have the graph $H$ obtained at the end being equal to $G$ with probability at least $1-\delta$. The sample complexity is dominated by the number of samples $N$ required for Sparsitron and the running time is at most $O(N\cdot n^t)$ where the $n^t$ dependence comes from the feature expansion and the evaluation of the median on each of the monomials. 
     \end{proof}

We now state and prove the theorem on learning a distribution $D_{\tilde{\psi}}$ that is close to \csmooth $D_\psi$ in KL divergence and TV distance.
\begin{theorem}
\label{thm:mrf_param_recovery}
    Let $C,\lambda>0$ and $0<\epsilon,\delta<1$. Let $D_{\psi}$ be a \csmooth $t$-MRF and suppose $\norm{\di \psi}_1\leq \lambda$ for all $i\in [n]$. Then, there exists an algorithm that draws $N=\frac{\lambda^2((nt)^{O(t)}\exp(O(C))\log(n/\delta\epsilon)}{\epsilon^8}$ samples from $D_\psi$ and runs in time equal to $O(N\cdot n^t)$ such that it outputs a $t$-MRF $D_{\tilde{\psi}}$ with factorization polynomial $\tilde{\psi}$ such that (1) $\norm{\psi-\tilde{\psi}}_1\leq \epsilon^2$, (2) $\kl(D_{\psi},D_{\tilde{\psi}})\leq 2\epsilon^2$, and (3)
         $\tv(D_{\psi},D_{\tilde{\psi}}) \leq {\epsilon}$.
    The algorithm succeeds with probability at least $1-\delta$.
\end{theorem}
\begin{proof}
    After doing a feature expansion, we run the algorithm from \Cref{thm:sparsitron} nodewise to recover polynomials $\{p_i\}_{i\in [n]}$ such that $\E_{X\sim \dmrf}[(\sigmoid(p_i(X))-\sigmoid(2\di\psi(x)))^2]\leq \epsilon'$ for all $i\in [n]$ with $\epsilon'\leq O(\epsilon^4\cdot \exp(-10C)\cdot (1/nt)^{2t+2})$. The sample complexity is  $N=\frac{\lambda^2((nt)^{O(t)}\exp(O(C))\log(n/\delta\epsilon)}{\epsilon^8}$. From \Cref{lem:polynomial_recovery}, we obtain polynomials $\{p_i\}_{i\in [n]}$ such that $\norm{p_i-\di\psi}_1\leq \frac{\epsilon}{n}$ for all $i\in [n]$. Construct a polynomial $\tilde{\psi}$ such that $\widehat{\tilde{\psi}}(S)=\widehat{p_i}(S\setminus\{i\})$ where $i$ is an arbitrary index in $S$. Observe that $\norm{\psi-\tilde{\psi}}_1\leq \epsilon^2$. Now, the theorem follows from \Cref{lem:param_rec_tv_distance}.
\end{proof}

We now argue that a random MRF drawn from the distribution over MRFs defined in \Cref{defn:random_MRF} satisfies the properties required for structure learning and parameter recovery in \Cref{thm:structure learning MRFs,thm:mrf_param_recovery} with appropriate choice of parameters.
\begin{lemma}
\label{lem:random_mrf_prop}
    For a graph $G$ of degree $d$ and $t>0$, let $\Gmrf$ be as defined in \Cref{defn:random_MRF} with Gaussian (or Rademacher) coefficients. Then, we have that with probability at least $1-O(1/n^t)$ over $\dmrf\sim \Gmrf$,
     \begin{enumerate}
         \item $\dmrf$ is (a) $\frac{\beta}{n^{5t/2}}$-identifiable in the Gaussian case, (b) $\frac{\beta}{d^{(t-1)/2}}$-identifiable in the Rademacher case, 
         \item for all $i\in [n]$, $\norm{\di\psi}_1$ is at most (1) $\beta\cdot d^{(t+1)/2}\sqrt{t\log n}$ in the Gaussian case (2) $\beta\cdot d^{(t+1)/2}$ in the Rademacher case, 
         \item $\dmrf$ is $O(\beta^2t+\beta t\sqrt{\log n})$-smooth.
     \end{enumerate} 
\end{lemma}
\begin{proof}
    We begin with the proof of (1). From standard Gaussian anticoncentration, we  have that $|\widehat{\psi}(S)|\geq \frac{\beta}{d^{(t-1)/2}}\cdot\eta$ with probability at least $1-O(\eta)$ for any set $S$ and $\dmrf\sim \Gmrf$. Choosing $\eta=O(1/n^{2t})$ and taking a union bound over all monomials, we obtain that with probability at least $1-O(1/n^t)$, it holds that $|\widehat{\psi}(S)|\geq \frac{\beta}{n^{5t/2}}$. The bound for the Rademacher case is direct as each variable is $\frac{\beta}{d^{(t-1)/2}}\pm 1$. We now prove (2). From standard Gaussian tail bound, we have that $\norm{X}_{\infty}\leq O(\sqrt{\log k})$ with probability at least $1-1/k$ for $X\sim \Gauss(0,1)^k$. Applying this to all monomials, we have that $\norm{\di\psi}_1\leq \beta\cdot d^{(t+1)/2}\sqrt{t\log n}$ for all $i\in [n]$ with probability at least $1-O(1/n^t)$. Finally for (3), we prove the following lemma. 
\begin{lemma}
    For a graph $G$ of degree $d$ and $t>0$, let $\Gmrf$ be as defined in \Cref{defn:random_MRF} with Gaussian (or Rademacher) coefficients. Then, we have that with probability at least $1-\frac{1}{(n+1)^t}$, $\dmrf\sim \Gmrf$ is $O(\beta^2t+\beta t\sqrt{\log n})$-smooth.
\end{lemma}
\begin{proof}
    We prove that $\Gmrf$ has subgaussian derivatives and then use \Cref{thm:subgauss_to_smooth}. Let $\psi$ be the random function corresponding to the factorization polynomial of an MRF sampled from \Cref{thm:subgauss_to_smooth}. Consider any $x,y\in \cube{n}$ with $x$ and $y$ differing in set $S$ and $|S|\leq t$. From \Cref{def:subgauss_derivatives}, it suffices to prove that the random variable $\psi(x)-\psi(y)$ is $\lambda$-subgaussian for appropriate $\lambda$. By the definition of the set $S$, we have that 
    \[
    \psi(x)-\psi(y)=2\sum_{\substack{T\in C_{t}(G)\\|T\cap S|\text{ is odd}}}\widehat{\psi}(T)\cdot \parity{T}(x).
    \]
    Recall that each coefficient of $\psi$ are independent and identically distributed to  $\frac{\beta}{\sqrt{d^t}}Z$ where $Z\sim \Gauss(0,1)$ or $Z\sim \Rademacher$. Note that the number of $T\in C_{t}(G)$ with which $S$ has non zero intersection is at most $t\cdot\sum_{i=0}^{t-1}\binom{d}{i}\leq O(t\cdot d^{t-1})$. 
    Since both $\Gauss(0,1)$ and $\Rademacher$ are $O(1)$-subgaussian, we have that $\psi(x)-\psi(y)$ is the sum of $O(t\cdot d^{t-1})$ iid random variables which are $\frac{2\beta}{d^{(t-1)/2}}$-subgaussian. Since the sum of $k$ $\lambda$-subgaussian variables is $O(\sqrt{k}\lambda)$-subgaussian (Proposition 2.6.1 from \cite{hdp}), we have that $\psi(x)-\psi(y)$ is $O(\beta\sqrt{t})$-subgaussian. Now, applying \Cref{thm:subgauss_to_smooth}, we obtain that with probability at least $1-\frac{1}{(n+1)^t}$, we have that $\dmrf\sim \Gmrf$ is $O(\beta^2t+\beta t\sqrt{\log n})$-smooth.
\end{proof}
Thus, we have proved that $\dpsi_{G,\beta,t}$ satisfies all the conditions that we need for efficient learnability. 
\end{proof}

The following theorems are now immediate from \Cref{thm:structure learning MRFs} and \Cref{lem:param_rec_tv_distance} and \Cref{lem:random_mrf_prop}. The first theorem states the sample complexity and running time for structure learning. 
\begin{theorem}[Structure Learning in Random MRFs]
\label{thm:structure_learning_random_mrf}
     For a graph $G$ of degree $d$ and $t>0$, let $\Gmrf$ be as defined in \Cref{defn:random_MRF}. Then, we have that with probability at least $1-O(1/n^t)$ over $\dmrf\sim \Gmrf$, there exists an algorithm 
  that draws $N=\frac{\exp(O(\beta^2t+\beta t\sqrt{\log n}))\cdot n^{O(t)}\cdot \log(1/\delta\beta)}{\beta^2}$ samples and recovers the graph $G$. The running time is $O(N\cdot n^t)$.

\end{theorem}
     Note that for $\beta\leq O(\sqrt{\log n})$ the above running time is polynomial in $n^{t}$. We now state the theorem of learning in TV distance. 
\begin{theorem}
\label{thm:mrf_tv_recovery}
     For a graph $G$ of degree $d$ and $t>0$, let $\Gmrf$ be as defined in \Cref{defn:random_MRF}. Then, we have that with probability at least $1-O(1/n^t)$ over $\dmrf\sim \Gmrf$, there exists an algorithm that draws $N=\frac{\exp(O(\beta^2t+\beta t\sqrt{\log n}))\cdot n^{O(t)}\cdot \log(1/\delta\epsilon)}{\epsilon^8}$ samples and runs in time $O(N\cdot n^t)$ that outputs a $t$-MRF $D_{\tilde{\psi}}$ such that (1) $\norm{\psi-\tilde{\psi}}_1\leq \epsilon^2$, (2) $\kl(D_{\psi},D_{\tilde{\psi}})\leq 2\epsilon^2$, and (3) $\tv(D_{\psi},D_{\tilde{\psi}})\leq {\epsilon}$.
\end{theorem}

For random MRFs with rademacher weights, we obtain improved bounds. We give an algorithm that draws a sub-polynomial number of samples and recovers the random MRF exactly. 
\begin{theorem}
\label{thm:random_mrf_param_rademacher}
     For a graph $G$ of degree $d$ and $t>0$, let $\Gmrf$ be as defined in \Cref{defn:random_MRF} with rademacher coefficients. Then, we have that with probability at least $1-O(1/n^t)$ over $\dmrf\sim \Gmrf$, there exists an algorithm that draws $N=\exp(O(\beta^2t+\beta t\sqrt{\log n}))\cdot O(td^{t}\log (n/\delta\beta)/\beta^2)$ samples and recovers $\psi$ exactly. The running time is $O(N\cdot n^t)$.
\end{theorem}
The proof is in \Cref{proof:random_mrf_param_rademacher}.
\begin{remark}
We remark that when the graph $G$ is known to the learner, the task of parameter recovery becomes  less expensive. When running Sparsitron for the index $i$, it suffices to only search over the polynomials whose variables are neighbours of $i$ in $G$. This is valid as the optimal polynomial $\di\psi$ only contains the neighbouring variables. Thus, the cost of feature expansion is now $d^{t}$ instead of $n^t$. Similarly, the error parameter in \Cref{lem:polynomial_recovery} now scales with $d^{t}$ instead of $n^t$. Thus, the sample complexity in \Cref{thm:mrf_tv_recovery} is now improved to  $N=\frac{\exp(O(\beta^2t+\beta t\sqrt{\log n}))\cdot d^{O(t)}\cdot n^8\log(n/\delta\epsilon)}{\epsilon^8}$ and the running time is now $O(N\cdot nd^t)$. Thus, for bounded $d$, the running time is now polynomial in $n$. 
\end{remark}

\bibliographystyle{alpha}
\bibliography{refs}

\newcommand{\etalchar}[1]{$^{#1}$}
\begin{thebibliography}{ZKKW20}

\bibitem[ABXY24]{adhikari2024spectral}
Arka Adhikari, Christian Brennecke, Changji Xu, and Horng-Tzer Yau.
\newblock Spectral gap estimates for mixed p-spin models at high temperature.
\newblock {\em Probability Theory and Related Fields}, pages 1--29, 2024.

\bibitem[AJ24]{arous2024shattering}
G{\'e}rard~Ben Arous and Aukosh Jagannath.
\newblock Shattering versus metastability in spin glasses.
\newblock {\em Communications on Pure and Applied Mathematics}, 77(1):139--176, 2024.

\bibitem[AJK{\etalchar{+}}24]{anari2024universality}
Nima Anari, Vishesh Jain, Frederic Koehler, Huy~Tuan Pham, and Thuy-Duong Vuong.
\newblock Universality of spectral independence with applications to fast mixing in spin glasses.
\newblock In {\em Proceedings of the 2024 Annual ACM-SIAM Symposium on Discrete Algorithms (SODA)}, pages 5029--5056. SIAM, 2024.

\bibitem[AKN06]{abbeel2006learning}
Pieter Abbeel, Daphne Koller, and Andrew~Y Ng.
\newblock Learning factor graphs in polynomial time and sample complexity.
\newblock {\em The Journal of Machine Learning Research}, 7:1743--1788, 2006.

\bibitem[AMS23]{alaoui2023shattering}
Ahmed~El Alaoui, Andrea Montanari, and Mark Sellke.
\newblock Shattering in pure spherical spin glasses.
\newblock {\em arXiv preprint arXiv:2307.04659}, 2023.

\bibitem[BB19]{bauerschmidt2019very}
Roland Bauerschmidt and Thierry Bodineau.
\newblock A very simple proof of the lsi for high temperature spin systems.
\newblock {\em Journal of Functional Analysis}, 276(8):2582--2588, 2019.

\bibitem[BGPV21]{bhattacharyya2021near}
Arnab Bhattacharyya, Sutanu Gayen, Eric Price, and NV~Vinodchandran.
\newblock Near-optimal learning of tree-structured distributions by chow-liu.
\newblock In {\em Proceedings of the 53rd annual acm SIGACT symposium on theory of computing}, pages 147--160, 2021.

\bibitem[BMS08]{bresler2008reconstruction}
Guy Bresler, Elchanan Mossel, and Allan Sly.
\newblock Reconstruction of markov random fields from samples: Some observations and algorithms.
\newblock In {\em International Workshop on Approximation Algorithms for Combinatorial Optimization}, pages 343--356. Springer, 2008.

\bibitem[Bre15]{Bresler15}
Guy Bresler.
\newblock Efficiently learning ising models on arbitrary graphs.
\newblock In Rocco~A. Servedio and Ronitt Rubinfeld, editors, {\em Proceedings of the Forty-Seventh Annual {ACM} on Symposium on Theory of Computing, {STOC} 2015, Portland, OR, USA, June 14-17, 2015}, pages 771--782. {ACM}, 2015.

\bibitem[BRO17]{bachschmid2017statistical}
Ludovica Bachschmid-Romano and Manfred Opper.
\newblock A statistical physics approach to learning curves for the inverse ising problem.
\newblock {\em Journal of Statistical Mechanics: Theory and Experiment}, 2017(6):063406, 2017.

\bibitem[BSXY24]{brennecke2024two}
Christian Brennecke, Adrien Schertzer, Changji Xu, and Horng-Tzer Yau.
\newblock The two point function of the sk model without external field at high temperature.
\newblock {\em Probab. Math. Phys}, 5:131--175, 2024.

\bibitem[BXY23]{brennecke2023operator}
Christian Brennecke, Changji Xu, and Horng-Tzer Yau.
\newblock Operator norm bounds on the correlation matrix of the sk model at high temperature.
\newblock {\em arXiv preprint arXiv:2307.12535}, 2023.

\bibitem[CH71]{hammersley-clifford}
P~Clifford and JM~Hammersley.
\newblock Markov fields on finite graphs and lattices.
\newblock 1971.

\bibitem[CL68]{chow1968approximating}
CKCN Chow and Cong Liu.
\newblock Approximating discrete probability distributions with dependence trees.
\newblock {\em IEEE transactions on Information Theory}, 14(3):462--467, 1968.

\bibitem[DDDK21]{dagan2021learning}
Yuval Dagan, Constantinos Daskalakis, Nishanth Dikkala, and Anthimos~Vardis Kandiros.
\newblock Learning ising models from one or multiple samples.
\newblock In {\em Proceedings of the 53rd Annual ACM SIGACT Symposium on Theory of Computing}, pages 161--168, 2021.

\bibitem[DKSS21]{diakonikolas2021outlier}
Ilias Diakonikolas, Daniel~M Kane, Alistair Stewart, and Yuxin Sun.
\newblock Outlier-robust learning of ising models under dobrushin’s condition.
\newblock In {\em Conference on Learning Theory}, pages 1645--1682. PMLR, 2021.

\bibitem[EA75]{edwards1975theory}
Samuel~Frederick Edwards and Phil~W Anderson.
\newblock Theory of spin glasses.
\newblock {\em Journal of Physics F: Metal Physics}, 5(5):965, 1975.

\bibitem[EAG24]{el2024bounds}
Ahmed El~Alaoui and Jason Gaitonde.
\newblock Bounds on the covariance matrix of the sherrington--kirkpatrick model.
\newblock {\em Electronic Communications in Probability}, 29:1--13, 2024.

\bibitem[EKZ22]{eldan2022spectral}
Ronen Eldan, Frederic Koehler, and Ofer Zeitouni.
\newblock A spectral condition for spectral gap: fast mixing in high-temperature ising models.
\newblock {\em Probability theory and related fields}, 182(3):1035--1051, 2022.

\bibitem[GJK23]{gamarnik2023shattering}
David Gamarnik, Aukosh Jagannath, and Eren~C K{\i}z{\i}lda{\u{g}}.
\newblock Shattering in the ising pure $ p $-spin model.
\newblock {\em arXiv preprint arXiv:2307.07461}, 2023.

\bibitem[GKK19]{goel2019learning}
Surbhi Goel, Daniel~M Kane, and Adam~R Klivans.
\newblock Learning ising models with independent failures.
\newblock In {\em Conference on Learning Theory}, pages 1449--1469. PMLR, 2019.

\bibitem[GM24]{gaitonde2024unified}
Jason Gaitonde and Elchanan Mossel.
\newblock A unified approach to learning ising models: Beyond independence and bounded width.
\newblock In {\em Proceedings of the 56th Annual ACM Symposium on Theory of Computing}, pages 503--514, 2024.

\bibitem[GMM24]{gaitonde2024efficiently}
Jason Gaitonde, Ankur Moitra, and Elchanan Mossel.
\newblock Efficiently learning markov random fields from dynamics.
\newblock {\em arXiv preprint arXiv:2409.05284}, 2024.

\bibitem[HKM17]{hamilton2017information}
Linus Hamilton, Frederic Koehler, and Ankur Moitra.
\newblock Information theoretic properties of markov random fields, and their algorithmic applications.
\newblock {\em Advances in Neural Information Processing Systems}, 30, 2017.

\bibitem[KKSK11]{kakade2011efficient}
Sham~M Kakade, Varun Kanade, Ohad Shamir, and Adam Kalai.
\newblock Efficient learning of generalized linear and single index models with isotonic regression.
\newblock {\em Advances in Neural Information Processing Systems}, 24, 2011.

\bibitem[KM17]{sparsitron}
Adam~R. Klivans and Raghu Meka.
\newblock Learning graphical models using multiplicative weights.
\newblock {\em 2017 IEEE 58th Annual Symposium on Foundations of Computer Science (FOCS)}, pages 343--354, 2017.

\bibitem[MM09a]{mezard2009information}
Marc Mezard and Andrea Montanari.
\newblock {\em Information, physics, and computation}.
\newblock Oxford University Press, 2009.

\bibitem[MM09b]{mezard2009constraint}
Marc M{\'e}zard and Thierry Mora.
\newblock Constraint satisfaction problems and neural networks: A statistical physics perspective.
\newblock {\em Journal of Physiology-Paris}, 103(1-2):107--113, 2009.

\bibitem[MMS21]{moitra2021learning}
Ankur Moitra, Elchanan Mossel, and Colin~P Sandon.
\newblock Learning to sample from censored markov random fields.
\newblock In {\em Conference on Learning Theory}, pages 3419--3451. PMLR, 2021.

\bibitem[MPV87]{mezard1987spin}
Marc M{\'e}zard, Giorgio Parisi, and Miguel~Angel Virasoro.
\newblock {\em Spin glass theory and beyond: An Introduction to the Replica Method and Its Applications}, volume~9.
\newblock World Scientific Publishing Company, 1987.

\bibitem[NBSS10]{netrapalli2010greedy}
Praneeth Netrapalli, Siddhartha Banerjee, Sujay Sanghavi, and Sanjay Shakkottai.
\newblock Greedy learning of markov network structure.
\newblock In {\em 2010 48th Annual Allerton Conference on Communication, Control, and Computing (Allerton)}, pages 1295--1302. IEEE, 2010.

\bibitem[Pan13]{panchenko2013sherrington}
Dmitry Panchenko.
\newblock {\em The sherrington-kirkpatrick model}.
\newblock Springer Science \& Business Media, 2013.

\bibitem[PSBR20]{prasad2020learning}
Adarsh Prasad, Vishwak Srinivasan, Sivaraman Balakrishnan, and Pradeep Ravikumar.
\newblock On learning ising models under huber's contamination model.
\newblock {\em Advances in neural information processing systems}, 33:16327--16338, 2020.

\bibitem[SW12]{santhanam2012information}
Narayana~P Santhanam and Martin~J Wainwright.
\newblock Information-theoretic limits of selecting binary graphical models in high dimensions.
\newblock {\em IEEE Transactions on Information Theory}, 58(7):4117--4134, 2012.

\bibitem[Tal03]{talagrand2003spin}
Michel Talagrand.
\newblock {\em Spin glasses: a challenge for mathematicians: cavity and mean field models}, volume~46.
\newblock Springer Science \& Business Media, 2003.

\bibitem[Tal10]{talagrand2010mean}
Michel Talagrand.
\newblock {\em Mean field models for spin glasses: Volume I: Basic examples}, volume~54.
\newblock Springer Science \& Business Media, 2010.

\bibitem[TR14]{tandon2014learning}
Rashish Tandon and Pradeep Ravikumar.
\newblock Learning graphs with a few hubs.
\newblock In {\em International conference on machine learning}, pages 602--610. PMLR, 2014.

\bibitem[Ver18]{hdp}
R.~Vershynin.
\newblock {\em High-Dimensional Probability: An Introduction with Applications in Data Science}.
\newblock Cambridge Series in Statistical and Probabilistic Mathematics. Cambridge University Press, 2018.

\bibitem[VML22]{vuffray2022efficient}
Marc Vuffray, Sidhant Misra, and Andrey~Y Lokhov.
\newblock Efficient learning of discrete graphical models.
\newblock {\em Journal of Statistical Mechanics: Theory and Experiment}, 2021(12):124017, 2022.

\bibitem[VMLC16]{VuffrayMLC16}
Marc Vuffray, Sidhant Misra, Andrey~Y. Lokhov, and Michael Chertkov.
\newblock Interaction screening: Efficient and sample-optimal learning of ising models.
\newblock In Daniel~D. Lee, Masashi Sugiyama, Ulrike von Luxburg, Isabelle Guyon, and Roman Garnett, editors, {\em Advances in Neural Information Processing Systems 29: Annual Conference on Neural Information Processing Systems 2016, December 5-10, 2016, Barcelona, Spain}, pages 2595--2603, 2016.

\bibitem[WLR06]{wainwright2006high}
Martin~J Wainwright, John Lafferty, and Pradeep Ravikumar.
\newblock High-dimensional graphical model selection using $\ell_1 $-regularized logistic regression.
\newblock {\em Advances in neural information processing systems}, 19, 2006.

\bibitem[WSD19]{wu2019sparse}
Shanshan Wu, Sujay Sanghavi, and Alexandros~G Dimakis.
\newblock Sparse logistic regression learns all discrete pairwise graphical models.
\newblock {\em Advances in Neural Information Processing Systems}, 32, 2019.

\bibitem[ZKKW20]{zhang2020privately}
Huanyu Zhang, Gautam Kamath, Janardhan Kulkarni, and Steven Wu.
\newblock Privately learning markov random fields.
\newblock In {\em International conference on machine learning}, pages 11129--11140. PMLR, 2020.

\end{thebibliography}
\newpage
\appendix
\section{Lemmas on Polynomial Recovery}
This section has some useful results on polynomial recovery in \csmooth MRFs. The proofs of these claims follow the steps of the analogous claims in Section 6 of \cite{sparsitron}, with appropriate changes to handle \csmooth distributions as opposed to the unbiased ones considered in \cite{sparsitron}. The property of unbiasedness is a worst case property, but we extend their proofs to work under the average case properties of \csmooth MRFs.

First, we prove that for any polynomial $p$ with maximal monomial $S$ and any $x$, there exists a vector $y$ at a Hamming distance at $|S|$ from $x$ such that $|p(x)|\geq |\widehat{p}(S)|$.
\begin{lemma}
    \label{lem:anticonc_dist}
    Let $p$ be a polynomial on $\cube{n}$ and let $S$ be a maximal monomial of $p$. For any $x\in\cube{n}$, there exists a vector $y\in \cube{n}$ with $\dhamming(x,y)\leq |S|$ such that $|p(y)|\geq |\widehat{p}(S)|$.
\end{lemma}
\begin{proof}
    We prove the lemma by induction on $|S|$. We first consider the base case $|S|=1$. Let $S=\{i\}$. We have that 
   \[p(x)=\widehat{p}\left(\{i\}\right)x_i+p^{-i}(x)\] where $p^{-i}(x)$ does not depend on $x_i$. Choosing $y$ such that $y_i=\sign\left(\widehat{p}\left({\{1\}}\right)\cdot p^{-1}(x)\right)$ and $y_j=x_j$ for $j\neq i$, we have that $|p(y)|\geq |\widehat{p}\left(\{1\}\right)|$. Clearly, $\dhamming(x,y)\leq 1$. 

    Say the claim is true for $|S|-1$. Consider a maximal monomial $S$ with $i\in S$ for some index $i$. We have that $p(x)=x_i\cdot \partial_{i}p(x)+p^{-i}(x)$ where $p^{-i}(x)$ and $\di p$ do not depend on $x_i$. Observe that $S\setminus\{i\}$ is a maximal monomial for $\di p$ with coefficient equal to $\widehat{p}(S)$. Thus, we have that there exists a vector $z$ with $\dhamming(x,z)\leq |S|-1$ and $z_i=x_i$ such that $|\di p(z)|\geq |\widehat{p}(S)|$. Construct $y$ such that $y_i=\sign(\di p(z)\cdot p^{-1}(z))$ and $y_j=z_j$ for $j\neq i$. Clearly, we have that $|p(y)|\geq |\di p(z)|\geq |\widehat{p}(S)|$ we also have that $\dhamming(x,y)\leq |S|$ as $y$ and $z$ differ in at most one coordinate. This completes the proof. 
\end{proof}

The following lemma proves that for for any $p$ such that $\E_{X\sim \dmrf}[\left(\sigma(p(X))-\sigma(2\di \psi(X))\right)^2]$ is small, it is possible to estimate the coefficient $\widehat{\di\psi}(S)$ using $p$ for any maximal monomial $S$.
\begin{lemma}
\label{lem:struct_recovery_polynomial}
    For $C>0$, let $\dmrf$ be a \csmooth $t$-MRF. Let $p$ be a polynomial and $i\in [n]$ such that
$\E_{X\sim \dmrf}[\left(\sigma(p(X))-\sigma(2\di \psi(X))\right)^2]\leq \epsilon$ for $\epsilon>0$. Then, for any maximal monomial $S$ of $\di\psi$ of size at most $t-1$, we have that 
\[
  \Pr_{X\sim \dmrf}\left[|\widehat{2\di\psi}({S})-\partial_{S}p(X)|>\delta|\right]\leq \frac{2^{t-1}\exp(10C+6)\cdot \epsilon}{\delta^2}  +\frac{1}{8}  
\]
\end{lemma}
\begin{proof}
    The proof follows the same structure as the proof of Lemma~{6.2} in \cite{sparsitron}. We make appropriate changes to account for the fact that we do not have unbiasedness anymore.  
    Recall the definition of $\cevent$ from \Cref{defn:csmooth}. We define the set \[\scevent\coloneq \{y\mid \text{ there exists $x\in \cevent$ with $x_{[n]\setminus S}=y$}\}.\]
    We first prove the following claim that $|\di\psi(x)|\leq 2C$ for $x\in \cube{n}$ with $x_{[n]\setminus S}\in \scevent$.
    \begin{claim}
    \label{claim: scevent_bound_factorization}
        For $y\in \cube{n}$ 
        with $y_{[n]\setminus S}\in \scevent$, it holds that $|\di\psi(y)|\leq 2C$.
    \end{claim}
\begin{proof}
    Since $y_{[n]\setminus S}\in \scevent$, we have that $y_{[n]\setminus S}=x_{[n]\setminus S}$ for $x\in \cevent$.  
    For $x\in \cevent$, we have that 
    \[
        |\di\psi(x)|=\frac{1}{2}|\psi^{\{i\}}(x)|\leq \frac{C}{2}
    \] from the fact that $\psi^{\{i\}}(x)=2x_i\cdot \di\psi(x)$. Thus, we have that 
    \begin{align*}
        |\di\psi(x)-\di\psi(y)|&=\frac{1}{2}|x_i\psi^{\{i\}}(x)-y_i\psi^{\{i\}}(y)|
        \leq \frac{1}{2}|\psi(x)-\psi(x^{\{i\}})|+\frac{1}{2}|\psi(y)-\psi(y^{\{i\}})|\\
        &\leq \frac{1}{2}\left(|\psi(x)-\psi(x^{\{i\}})|+|\psi(y)-\psi(x)|+|\psi(y^{\{i\}})-\psi(x)|\right)
        \leq \frac{3C}{2}.
    \end{align*} 
    The first inequality follows from the definition of $\psi^{\{i\}}$ and a triangle inequality . The second inequality again follows from a triangle inequality. The last inequality follows from the facts that $x\in \cevent$ and $\max\left(\dhamming(x,x^{\{i\}}),\dhamming(x,y),\dhamming(x,y^{\{i\}})\right)\leq t$. Thus, we obtain that $|\di\psi(y)|\leq 2C$.
\end{proof}
From \Cref{fact:antilipschitz_sigmoid}, we have that for any $y$ with $y_{[n]\setminus S}\in \scevent$,
\[
|\sigma(p(y))-\sigma(2\di \psi(y))|\geq \exp(-4C-3)\cdot\min(1,|p(X)-2\di \psi(X)|)   . 
\] since  $|\di\psi(y)|\leq 2C$ for all $y$ with $y_{[n]\setminus S}\in \scevent$. Squaring and taking expectation, we obtain 
\begin{align*}
    \E_{X\sim \dmrf}&\left[\min(1,\diffpoly^2)\cdot \1\{X_{[n]\setminus S}\in \scevent\}\right]\\
    &\leq \exp(8C+6) \E_{X\sim \dmrf}\left[\diffsigmoid\cdot \1\{X_{[n]\setminus  S}\in \scevent\}\right]\leq \exp(8C+6)\cdot\epsilon
\end{align*}
From Markov's inequality, for any $0<\delta<1$, we have that 
\begin{align*}
    \Pr_{X\sim \dmrf}\left[\diffpoly\geq \delta, X_{[n]\setminus S}\in \scevent\right]\leq \frac{\exp(8C+6)\cdot \epsilon}{\delta^2}.
\end{align*}
Similar to the proof of Lemma~{6.2} in \cite{sparsitron}, for any fixing $z$ of variables not in $S$, let $r_z(x_S)$ be the polynomial obtained from $p(x)-2\di\psi(x)$ by fixing the variables outside $S$ to $z$. Note that $\widehat{r_z}(S)=\widehat{2\di\psi}(S)-\partial_Sp(X)$ where $X_{[n]\setminus S}=z$ as $S$ is maximal in $\di\psi$. 

For any $x_{[n]\setminus S}\in \scevent$ with $x_{[n]\setminus S}=z$, \Cref{lem:anticonc_dist} implies that there exists $y$ with $\dhamming(x,y)\leq t-1$ and $y_{[n]\setminus S}=x_{[n]\setminus S}$ such that $|r_z(y_S)|\geq |\widehat{r_z}(S)|$. Let $w\in \cube{n}$ such that $\Pr_{X\sim \dmrf}\left[X=w\mid X_{[n]\setminus s}=z\right] \geq 2^{-t+1}$. There always exists one such vector as it is fixed in all but $t-1$ indices. We have that for $z\in \scevent$,
\begin{align}
\label{eqn: polynomial_anticonc}
    \Pr_{X\sim \dmrf}\left[|r_{z}(X_S)|\geq |\widehat{r_z}(S)|,\mid X_{[n]\setminus S}=z\right]&=
    \Pr_{X\sim \dmrf}\left[X=y\mid X_{[n]\setminus S}=z\right]\nonumber\\
    &=\frac{\Pr_{X\sim \dmrf}[X=y]}{\Pr_{X\sim \dmrf}[X=w] }\cdot \frac{\Pr_{X\sim \dmrf}\left[X=w\right]}{\Pr_{X\sim \dmrf}\left[X_{[n]\setminus s}=z\right]}\nonumber\\
    &\geq \exp(\psi(y)-\psi(w))\cdot 2^{-t+1}
    \geq \exp(-2C)\cdot 2^{-t+1}
\end{align}
where the first equality follows from the definition of conditional probability. The first inequality follows from the definition of $w$ and the last inequality follows from the definition of $\scevent$.

We now have that 
\begin{align*}
   &\Pr_{X\sim \dmrf}\left[\diffpoly\geq \delta,X_{[n]\setminus S}\in \scevent\right]\\
    &\geq\sum_{z\in \scevent}\Pr_{X\sim \dmrf}\left[|p(X)-2\di\psi(X)|\geq \delta\mid |\widehat{2\di\psi}(S)-\partial_Sp(X)|\geq \delta, X_{[n]\setminus S}=z\right] \\
    &\cdot \Pr_{X\sim \dmrf}\left[|\widehat{2\di\psi}(S)-\partial_Sp(X)|\geq \delta, X_{[n]\setminus S}=z\right]\\
    &\geq \sum_{z\in \scevent}\exp(-2C)\cdot 2^{-t+1}\cdot \Pr_{X\sim \dmrf}\left[|\widehat{2\di\psi}(S)-\partial_Sp(X)|\geq \delta, X_{[n]\setminus S}=z\right]\\
    &\geq \exp(-2C)\cdot 2^{-t+1}\Pr_{X\sim \dmrf}\left[|\widehat{2\di\psi}(S)-\partial_Sp(X)|\geq \delta, X_{[n]\setminus S}\in \scevent\right]
\end{align*}

Putting things together, we obtain that 
\[
    \Pr_{X\sim \dmrf}\left[|\widehat{2\di\psi}(S)-\partial_Sp(X)|\geq \delta, X_{[n]\setminus S}\in \scevent\right]\leq \frac{2^{t-1}\exp(10C+6)\cdot \epsilon}{\delta^2}    
\]

Now, we have that 
\[
    \Pr_{X\sim \dmrf}\left[|\widehat{2\di\psi}(S)-\partial_Sp(X)|\geq \delta\right]\leq \Pr_{X\sim \dmrf}\left[|\widehat{2\di\psi}(S)-\partial_Sp(X)|\geq \delta, X_{[n]\setminus S}\in \scevent\right]+\frac{1}{8}    
\] where we used the fact that $\Pr_{X\sim \dmrf}[X_{[n]\setminus S}\in \scevent]\geq \Pr_{X\sim \dmrf}[X\in \cevent]\geq \frac{7}{8}$. Combining the two equations above completes the proof.
\end{proof}

The following lemma asserts that for a \csmooth MRF, the error $\norm{p-2\di \psi}_1\leq O\left(e^{O(C)}\binom{n}{t}\epsilon\right)$ whenever $\E_{X\sim \dmrf}\left[(\sigmoid(p(X))-\sigmoid(2\di\psi(X)))^2\right]\leq \epsilon$. This lemma generalizes \Cref{lem:gaitonde_recovery} from \cite{gaitonde2024unified}. The proof closely follows the proof of Lemma~{6.4} in \cite{sparsitron}. Their lemma also recovers the external fields, contrary to a remark in \cite{gaitonde2024unified}.
\begin{lemma}
\label{lem:polynomial_recovery}
    For $C>0$, let $\dmrf$ be a \csmooth $t$-MRF. Let $p$ be a polynomial and $i\in [n]$ such that
$\E_{X\sim \dmrf}[\left(\sigma(p(X))-\sigma(2\di \psi(X))\right)^2]\leq \epsilon$ for $0<\epsilon\leq \exp(-2C)\cdot 2^{-t}$. Then, we have that \[
\norm{p-2\di\psi}_1\leq  O(1)(4t)^t\binom{n}{t}\exp(5C)\sqrt{\epsilon}.\]
\end{lemma}
\begin{proof}
    This proof follows closely the proof of Lemma~6.4 in \cite{sparsitron}. Most of the details are the same, except that we make appropriate changes to handle the fact that we do not have unbiasedness, similar to \Cref{lem:struct_recovery_polynomial}. We borrow notation from the proof of the aforementioned lemma in \cite{sparsitron}. We highlight the parts where we make changes and do not re-derive the steps that are exactly the same. Let $r=p-2\di\psi$ be the difference polynomial. For $\ell\leq t-1$, let $r_{=\ell}$ be the polynomial obtained from $r$ by only considering monomials of size exactly equal to $\ell$. For $\ell\leq t-1$, let $\rho_{\ell}=\norm{r_{=\ell}}_1$. Clearly, the quantity we want to bound is $\norm{r}_1=\sum_{i=0}^{t-1}\rho_{i}$. We bound $\rho_0,\ldots \rho_{t-1}$ inductively, starting with $\rho_{t-1}$. We first, bound $\rho_{t-1}$.
    \begin{claim}
    \label{claim:rec_max_monomial}
    Consider any maximal monomial $S$ of $r=p-2\di\psi$ for $p$ satisfying the assumptions of \Cref{lem:polynomial_recovery}. Then, it holds that $|\widehat{r}(S)|\leq \exp(5C+3)\cdot 2^{t/2}\sqrt{\epsilon}$.
    \end{claim}
    \begin{proof}
Recall the definition of $\scevent$. From the fact that $\Pr_{X\sim \dmrf}[X_{[n]\setminus S}\in \scevent]\geq \frac{7}{8}$ and the assumption of the lemma, we obtain (using an averaging argument) that there exists a vector $z\in \scevent$ such that 
    \[
    \E_{X\sim \dmrf}\left[(\sigma(p(X))-\sigma(2\di\psi(X)))^2\mid X_{[n]\setminus S}=z\right]\leq 2\epsilon.
    \]

    Similar to the proof of \Cref{lem:struct_recovery_polynomial}, let $r_{z}$ be the polynomial obtained from $r$ by fixing the variables in $[n]\setminus S$ to $z$. Note that $\widehat{r_z}(S)=\widehat{r}(S)$. From the argument preceding \Cref{eqn: polynomial_anticonc}, we have that \[\Pr_{X\sim \dmrf}\left[|r_z(X)\geq |\widehat{r_{z}}(S)|\mid X_{[n]\setminus S}=z\right]\geq \exp(-2C)\cdot 2^{-t+1}.\]
    Recall from \Cref{claim: scevent_bound_factorization} that $|\di\psi(y)|\leq 4C$ for any $y$ with $y_{[n]\setminus S}\in \scevent$. This fact, together with \Cref{fact:antilipschitz_sigmoid}, implies that
    \[
    2\epsilon\geq \E_{X\sim \dmrf}\left[\left(\sigma(p(X))-\sigma(2\di\psi(X))\right)^2\mid X_{[n]\setminus S}=z\right]
    \geq \exp(-10C-6)\cdot 2^{-t+1}\cdot \min\left(1, \widehat{r}(S)\right)^2\] 
    From the bound $\epsilon\leq \frac{1}{2}\exp(-5C-3)\cdot 2^{-t+1}$, we have that $\widehat{r}(S)\leq 1$ which implies $|\widehat{r}(S)|\leq \exp(5C+3)\cdot 2^{t/2}\sqrt{\epsilon}$. 
    \end{proof}
    Since all degree $t-1$ terms are maximal, the above claim implies that $\rho_{t-1}\leq \binom{n}{t-1}\exp(5C+3)\cdot 2^{t/2}\sqrt{\epsilon}$. We now do the inductive step. For $|I|=\ell<t-1$, again by an averaging argument, we have $z$ such that $\E_{X\sim \dmrf}\left[\left(\sigma(p(X))-\sigma(2\di\psi(X))\right)^2\mid X_{[n]\setminus I}=z\right]\leq 2\epsilon$. Let $r_{z}$ be the polynomial obtained from $r$ by setting these variables to $z$. Again, repeating the same steps as before, we obtain that $\widehat{r_z}(I)\leq \exp(5C+3)2^{t/2}\sqrt{\epsilon}$. From here, the proof is exactly identical to the proof in \cite{sparsitron} except we set $\epsilon_0\equiv \exp(5C+3)2^{t/2}\sqrt{\epsilon}\cdot \binom{n}{t-1}$. Their analysis yields that 
    \[
    \norm{r}_{1}\leq 2^{t}t^{t}\cdot \epsilon_0\leq O(1)(4t)^t\binom{n}{t}\exp(5C)\sqrt{\epsilon}.
    \]
\end{proof}

\section{Parameter Recovery vs TV Distance}
We now argue that an MRF $\widehat{D}$ that has parameters close to $\dmrf$ has low TV distance (in fact KL divergence) with respect to $\dmrf$.
First we define {KL} divergence and {TV} distance.
\begin{definition}[KL Divergence]
Let $P$ and $Q$ be two distributions. The \textit{KL}-divergence between $P$ and $Q$ is defined as 
\[
\kl(P,Q)\coloneq\E_{X\sim P}[\log(P(X)/Q(X))]
\]
\end{definition}
We now define the Total Variation distance. 
\begin{definition}[TV Distance]
Let $P$ and $Q$ be two distributions. The \textit{KL}-divergence between $P$ and $Q$ is defined as 
\[
\tv(P,Q)\coloneq\frac{1}{2}\sum_{x\in \cube{n}}|P(x)-Q(x)|.
\]
\end{definition}
We are now ready to argue that parameter recovery in $t$-MRFs implies closeness in KL divergence/TV distance. The proof is almost the same as Lemma~3.6 in \cite{gaitonde2024unified}.
\begin{lemma}
\label{lem:param_rec_tv_distance}
    Let $D_{\psi}$ and $D_{\tilde{\psi}}$ be $t$-MRFs with factorization polynomials $\psi$ and $\tilde{\psi}$ respectively such that $\norm{\psi-\tilde{\psi}}_1\leq \epsilon$. Then, (1) $\kl(D_{\psi},D_{\tilde{\psi}})\leq 2\epsilon$, and (2) $\tv(D_{\psi},D_{\tilde{\psi}})\leq \sqrt{\epsilon}$.
\end{lemma}
\begin{proof}
    We prove (1). (2) follows from Pinsker's inequality. Let $Z_\psi$ and $Z_{\tilde{\psi}}$ are the partition functions of $D_{\psi}, D_{\tilde{\psi}}$ respectively. For any $x\in \cube{n}$, we have that 
    $\exp(\psi(x))\leq \exp(\epsilon)\cdot \exp(\tilde{\psi}(x))$ and  $\exp(\tilde{\psi}(x))\leq \exp(\epsilon)\cdot \exp({\psi}(x))$. From the definition of the partition function, we have that $Z_{\psi}\leq \exp(\epsilon)\cdot Z_{\tilde{\psi}}$ and $Z_{\tilde{\psi}}\leq \exp(\epsilon)\cdot Z_{\psi}$. For any $x\in \cube{n}$, we thus have that $\frac{D_{\psi}(x)}{D_{\tilde{\psi}(x)}}\leq \exp(2\epsilon)$. Now, from the definition of KL divergence, we get that $\kl(D_{\psi},D_{\tilde{\psi}})\leq 2\epsilon$.
\end{proof}
\section{Parameter Recovery in the Pure $t$-spin Model}
\begin{definition}[Pure $t$-spin model]
    A pure $t$-spin model is the distribution of $t$-MRFs such that that factorization polynomials $\psi$ is a random variable of the form 
    \[
    \psi(x)=\frac{\beta}{n^{(t-1)/2}}\sum_{(i_1,\ldots, i_t)\in [n]^{t}}\Gauss(0,1)\cdot \prod_{j=1}^{t}x_{i_j}.
    \]
\end{definition}

We now prove that pure $t$-spin models satisfy the properties required for parameter recovery in \Cref{thm:mrf_param_recovery}.
\begin{lemma}
    Let $\dpsi$ be the pure $t$-spin model with inverse temperature $\beta$. Then, it holds with probability at least $1-O(1/n^t)$ over $\dmrf\sim \dpsi$ that 
    \begin{enumerate}
        \item $\norm{\di\psi}_1\leq \beta\sqrt{t n^{t+1}\log n}$ for all $i\in [n]$,
        \item $\dmrf$ is $O(\beta^2t^2+\beta t\sqrt{t\log n})$-smooth.
    \end{enumerate}
\end{lemma}
\begin{proof}
    We first prove (1). The number of tuples $(i_1,i_2,\ldots i_t)$ is at most $n^{t}$. By standard gaussian tails, the maximum absolute values of all these coefficients is at most $\frac{\beta}{n^{(t-1)/2}}\sqrt{t\log n}$ with probability at least $1-O(1/n^t)$. Thus, we obtain $\norm{\di\psi}_1\leq \beta\sqrt{t n^{t+1}\log n}$. To prove (2), we argue that $\dpsi$ has subgaussian derivatives and then use \Cref{thm:subgauss_to_smooth}. For any multiset $\alpha\in [n]^t$ and set $S\subseteq[n]$, we say that $|S\cap \alpha|$ is odd if the number of common elements between $S$ and $t$ (counting repetitions) is odd. Observe that for any set $S$ with $|S|\leq t$ and any $x,y\in \cube{n}$ with $x$ and $y$ differing in the set $S$, we have that 
    \[
    \psi(x)-\psi(y)=\frac{2\beta}{n^{(t-1)/2}}\sum_{\substack{\alpha\in [n]^t\\ |S\cap \alpha| \text{ is odd} }}\Gauss(0,1)\cdot \prod_{i}^{t}x_{\alpha_i}.
    \]
    We now count the number of terms in the above expression as that determines the subgaussianity of $\psi(x)-\psi(y)$. Since $S$ has at most $t$ terms and the intersection is at least $1$, the number of terms is upper bounded by $|\bigcup_{i=1}^{t}\{\alpha \mid \alpha_i\in S, \alpha\in [n]^t\}|\leq t^2\cdot n^{t-1}$. Thus, we have that $\psi(x)-\psi(y)$ is $O(\beta t)$-subgaussian. Now, from \Cref{thm:subgauss_to_smooth}, we have that $\dmrf$ is $O(\beta^2t^2+\beta t\sqrt{t\log n})$-smooth. 
\end{proof}
Note that the smoothness is a factor of $t$ worse than what was obtained in \Cref{lem:random_mrf_prop}. This is because in the definition of the pure $t$-spin model, the same set can be counted multiple times whereas this was not allowed in the definition of the random MRF (\Cref{defn:random_MRF}). The following theorem on parameter recovery of these models is now immediate from \Cref{thm:mrf_param_recovery}.
\begin{theorem}
\label{thm:pure_tspin_recovery}
    Let $\dpsi$ be a pure $t$-spin model with inverse temperature $\beta$. With probability at least $1-O(1/n^t)$ over $\dmrf\sim \dpsi$, there exists an algorithm that draws $N=\frac{\exp(O(\beta^2t^2+\beta t\sqrt{t\log n}))\cdot n^{O(t)}\cdot \log(1/\delta\epsilon)}{\epsilon^8}$ samples and runs in time $O(N\cdot n^t)$ that outputs a $t$-MRF $D_{\tilde{\psi}}$ such that (1) $\norm{\psi-\tilde{\psi}}_1\leq \epsilon^2$, (2) $\kl(D_{\psi},D_{\tilde{\psi}})\leq 2\epsilon^2$, and (3) $\tv(D_{\psi},D_{\tilde{\psi}})\leq {\epsilon}$.
\end{theorem}
\section{Improved bounds for Rademacher Random MRFs}
\label{sec:random_mrf_rademacher}
For our algorithm with improved sample complexity to exactly learn the random MRF with Rademacher weights, we require a slightly modified version of \Cref{thm:sparsitron}.
\begin{theorem}
\label{thm:sparsitron_known_params}
    Let $\lambda,\epsilon,\delta>0$. Let $D$ be a distribution on $\cube{n}\times \cube{}$ where $\Pr[Y=+1|X]=\sigmoid(w\cdot X+g(X))$ for $(X,Y)\sim D$ where  $w\in \R^{n}$ is an unknown vector with $\norm{w}_1\leq \lambda$ and $g:\cube{n}\to \R$ is a known function.  There exists an algorithm that takes $N=O\left(\lambda^2(\ln(n/\delta\epsilon))/\epsilon^2\right)$ independent samples from $D$, runs in time $O(nN)$, and outputs a vector $\widehat{w}$ such that
    \[
    \E_{(X,Y)\sim D}\left[(\sigmoid(w\cdot X+g(X))-\sigmoid(\widehat{w}\cdot X+g(X))^2\right]\leq \epsilon
    \] with probability at least $1-\delta$.
\end{theorem}
Note that the only difference between the above theorem and \Cref{thm:sparsitron} is the addition of the known function $g(X)$ to the conditional probability. The proof of the above theorem is almost identical to that of \Cref{thm:sparsitron} (also Theorem~3.1 in \cite{sparsitron}) with very few additional changes (we change one line in their algorithm). We descrive the change below. We borrow the notation from the proof of \cite{sparsitron} and only highlight key changes. 
\begin{proof}[Proof of \Cref{thm:sparsitron_known_params}]
We now use $u$ instead of $\sigma$ to refer to the link function. We apply the transformation $(x,y)\rightarrow (x,\frac{y+1}{2})$ so that the $+1$ labels are mapped to $+1$ and the $-1$ labels are mapped to $0$. From now on, we assume that $D$ is the distribution of inputs after this transformation. Thus, we have that $\E_{(X,Y)\sim D}[Y\mid X=x]=u(w\cdot X+g(X))$. 

The only changes we make to their algorithm are the following: (1) in line 4 of Algorithm~2 of \cite{sparsitron}, we redefine the loss vector $\ell^t$ to now be $\ell^t\coloneq (1/2)(\boldone+(u(\lambda p^t\cdot x^t+g(x^t))-y^t)x^t)$, and (2) in line $7$ we compute the empirical risk as $\hat{\varepsilon}(\lambda p^t)=(1/M)\sum_{j=1}^{M}(u(\lambda p^t\cdot a^t+g(a^t))-b^j)^2$. Note the addition of the term $g(x^t)$ in both the steps. We can do this as we know the function $g$. 

Now, we highlight the changes in the analysis. The steps of the argument until Equation~3.3 of the proof of \cite{sparsitron} are identical as the new loss vector $\ell^t$ is still a vector in $[0,1]^n$. The only change is in how we bound $\E_{(x^t,y^t)}[Q^t\mid (x^1,y^1),\ldots, (x^{t-1},y^{t-1})]$. We have that 
\begin{align*}
    \E_{(x^t,y^t)}&[Q^t\mid (x^1,y^1),\ldots, (x^{t-1},y^{t-1})]=\E_{(x^t,y^t)}[(p^t-(1/\lambda)w)\cdot \ell^t]\\
    &=(1/2)\E_{(x^t,y^t)}[(p^t-(1/\lambda)w)\cdot (u(\lambda p^t\cdot x^t+g(x^t))-y^t)\cdot x^t]\\
    &=(1/2\lambda)\E_{x^t}[(\lambda p^t\cdot x^t+g(x^t)-w\cdot x^t-g(x^t))(u(\lambda p^t\cdot x^t+g(x^t))-u(w\cdot x^t+g(x^t)))]\\
    &\geq (1/2\lambda)\E_{x^t}[(u(\lambda p^t\cdot x^t+g(x^t))-u(w\cdot x^t+g(x^t)))^2]=(1/2\lambda)\varepsilon(\lambda p^t)
\end{align*}
where $\varepsilon(v)\coloneq \E_{(X,Y)\sim D}[(u(v\cdot X+g(X))-u(w\cdot X+g(X)))^2]$ is the risk. The main difference from the proof of \cite{sparsitron} is the third equation where we add and subtract $g(x^t)$ and then use the lipschitzness property. The rest of the proof is exactly identical.
\end{proof}

We are now ready to prove \Cref{thm:random_mrf_param_rademacher}. 
\begin{proof}
    \label{proof:random_mrf_param_rademacher}
    Recall that $\dmrf\sim \dpsi_{G,\beta,t}$ (with rademacher weights) is $C$-smooth with $C=O(\beta^2t+\beta t\sqrt{\log n})$ with probability at least $1-O(1/n^t)$. We henceforth assume $\dmrf$ is $C$-smooth.
    We recover the coefficients of $\psi$ iteratively, starting with the degree $t$ terms and proceeding downwards. We use fresh samples per iteration. We first show the base case of recovering degree $t$ terms. For each $i\in [n]$, we use \Cref{thm:sparsitron} to find polynomials $\{p^{t}_i\}_{i\in [n]}$ such that 
    \[
    \E_{X\sim D_{\psi}}[(\sigma(p^{t}_i(X))-\sigma(2\di\psi(X)))^2]\leq \exp(-10C-6)\cdot(\beta^2/(16(2d)^t))
    \] 

    Now, from \Cref{claim:rec_max_monomial}, we have that $|\widehat{p^t_i}(S)-\widehat{2\di\psi}(S)|< \frac{\beta}{4d^{(t-1)/2}}$ for any maximal monomial $S$ of $p^{t}_i-2\di\psi$. To recover the coefficient $\hat{\psi}(S)$ for $|S|=t$, we consider any $i\in S$. Now, we have that $\widehat{\di\psi(X)}(S\setminus\{i\})=\widehat{\psi}(S)$. Note that $S\setminus \{i\}$ is a maximal monomial of $p^{t}_i-2\di\psi$ as it has degree $t-1$ which is the degree of the polynomial. Thus, it holds that $|\widehat{p^t_{i}}(S\setminus \{i\})/2-\widehat{\psi}(S)|\leq \frac{\beta}{2d^{(t-1)/2}}$. To  obtain $\widehat{\psi}(S)$ exactly, we round $\widehat{p^t_{i}}(S\setminus \{i\})/2$ to the nearest multiple of $\beta/d^{(t-1)/2}$. Thus, we have obtained all the degree $t$ coefficients of $\psi$ exactly. The sample complexity of this step is $N^{t}=\exp(O(\beta^2t+\beta t\sqrt{\log n}))\cdot O(d^{t}\log (n/\delta\beta)/\beta^2)$ and follows from \Cref{thm:sparsitron}.

    We now describe how to obtain coefficients of degree $j$ if we know $\widehat{\psi}(S)$ exactly for all $|S|>j$. Construct a function $g^{j}:\cube{n}\to\cube{}$ such that $g^j(x)=\sum_{|T|>j}\widehat{\psi}(T)\parity{T}(x)$. Let the polynomial $g^j_i$ be defined as $g^j_i=2\di g^j$. We can construct these polynomials as we know all coefficients of size greater than $j$. Now, for each $i\in [n]$, we find polynomials $\{p^{j}_i\}_{i\in [n]}$ of degree at most $j$ such that 
    \[
    \E_{X\sim D_{\psi}}[(\sigma(p^{j}_i(X)+g^{j}_i(X))-\sigma(2\di\psi(X)))^2]\leq \exp(-10C-6)\cdot(\beta^2/(16(2d)^t)).
    \] 
We note that $2\di\psi(X)=2\sum_{|S|\leq j}\widehat{\di\psi}(S)\parity{S}(X)+g^j_i(X)$. Thus, we can find the above polynomials by $\{p^j_i\}_{i\in [n]}$ by running the modified Sparsitron algorithm from \Cref{thm:sparsitron_known_params} (with known function $g^j_i$) after expanding the features to contain all monomials of degree at most $j$. Observe that the degree $j$ monomials in the polynomial $r=p^{j}_i+g^{j}_i-2\di\psi$ are maximal as all higher degree monomials are $0$. Thus, again, we use \Cref{claim:rec_max_monomial} and repeat the argument from the base case to obtain that $|\widehat{p^j_{i}}(S\setminus \{i\})/2-\widehat{\psi}(S)|\leq \frac{\beta}{2d^{(t-1)/2}}$ for all $|S|=j$. To obtain $\widehat{\psi}(S)$ exactly, we again round $\widehat{p^j_{i}}(S\setminus \{i\})/2$ to the nearest multiple of $\beta/d^{(t-1)/2}$. In this way iterating $j=t,t-1,\ldots 1$, we obtain, all the coefficients of $\psi$ exactly. Since we use fresh samples in each iteration, we pay a multiplicative factor of $t$ in the final sample complexity
    
\end{proof}
\end{document}